\documentclass[a4paper]{article}

\RequirePackage{amsthm,amsmath,amsfonts,amssymb}
\RequirePackage[numbers,sort&compress]{natbib}
\RequirePackage[colorlinks,citecolor=blue,urlcolor=blue]{hyperref}
\RequirePackage{graphicx}
\usepackage{fullpage}
\usepackage{soul}
\usepackage[normalem]{ulem}
\usepackage[utf8]{inputenc} 
\usepackage[T1]{fontenc}    
\usepackage{authblk}

\theoremstyle{plain}

\newtheorem{theorem}{Theorem}[section]
\newtheorem{corollary}[theorem]{Corollary}
\newtheorem{lemma}[theorem]{Lemma}

\newtheorem{assumption}[theorem]{Assumption}
\newtheorem{proposition}[theorem]{Proposition}
\newtheorem{defn}[theorem]{Definition}

\def\1{\bm{1}}

\def\rv{{\textnormal{v}}}

\DeclareMathAlphabet{\mathsfit}{\encodingdefault}{\sfdefault}{m}{sl}
\SetMathAlphabet{\mathsfit}{bold}{\encodingdefault}{\sfdefault}{bx}{n}

\def\0{{\bf 0}}
\def\1{{\bf 1}}

\def\FM{{\mathcal F}}

\def\OM{{\mathcal O}}

\def\RB{{\mathbb R}}
\def\EB{{\mathbb E}}

\def\PB{{\mathbb P}}

\def\xii{\mbox{\boldmath$\xi$\unboldmath}}

\def\argmin{\mathop{\rm argmin}}

\def\diag{\mathrm{diag}}

\newcommand{\Var}{\mathrm{Var}}

\newcommand{\red}{\textcolor{red}}
\def\cadlag{\text{c\`adl\`ag}}
\providecommand{\red}[1]{\color{red}}
\def\rv{\color{black}}

\begin{document}

\title{Limit Theorems for Stochastic Gradient Descent with Infinite Variance}

\author[1]{Jose Blanchet} 
\author[2]{Aleksandar Mijatovi\'c}
\author[1]{Wenhao Yang}
\affil[1]{{\normalsize Management Science and Engineering, Stanford University}} 
\affil[2]{{\normalsize Department of Statistics, University of Warwick}} 

\maketitle


\begin{abstract}
Stochastic gradient descent is a classic algorithm that has gained great popularity especially in the last decades as the most common approach for training models in machine learning. While the algorithm has been well-studied when stochastic gradients are assumed to have a finite variance,  there is significantly less research addressing its theoretical properties in the case of infinite variance gradients. In this paper, we establish the asymptotic behavior of stochastic gradient descent in the context of infinite variance stochastic gradients, assuming that the stochastic gradient is regular varying with index $\alpha\in(1,2)$. The closest result in this context was established in 1969 \cite{krasulina1969stochastic}, in the one-dimensional case and assuming that stochastic gradients belong to a more restrictive class of distributions. We extend it to the multidimensional case, covering a broader class of infinite variance distributions. As we show, the asymptotic distribution of the stochastic gradient descent algorithm can be characterized as the stationary distribution of a suitably defined Ornstein-Uhlenbeck process driven by an appropriate stable L\'evy process. Additionally, we explore the applications of these results in linear regression and logistic regression models.
\end{abstract}

\section{Introduction}
\label{sec: introduction}
Stochastic gradient descent (SGD) and its variants constitute a class of algorithms rooted in stochastic approximation methods \cite{robbins1951stochastic, kiefer1952stochastic} and have demonstrated significant success in training models for various machine learning applications \cite{duchi2011adaptive, kingma2014adam}. This empirical success has motivated extensive research into the theoretical properties of SGD, including finite-sample and asymptotic analyses, as well as algorithmic improvements \cite{robbins1971convergence, bertsekas2000gradient, kushner1979rates, pelletier1998almost, moulines2011non, pelletier1998weak, mou2022optimal}. However, most of these advancements assume stochastic gradients with finite variance. In this work, we aim to extend the theoretical understanding of SGD by examining its behavior under stochastic gradients with infinite variance. To set the stage for this investigation, we first review the standard formulation of SGD. Consider the following problem:
\begin{align*}
    \min_{\theta\in\mathbb{R}^d}\mathbb{E}_{\xii\sim P}[\ell(\theta,\xii)],
\end{align*}
where $\xii:\Omega\to E$  is a random variable on probability space $(\Omega,\mathbb{P},\mathcal{F})$ and $\ell:\mathbb{R}^d\times E\to\mathbb{R}$ is a strongly convex function. The SGD algorithm is an iterative scheme to find the near optimal solution $\theta^*\in\argmin_{\theta\in\mathbb{R}^d}\mathbb{E}_{\xii\sim P}[\ell(\theta,\xii)]$:
\begin{align}
\label{eq: sgd_theta}
    \theta_{n+1}=\theta_n-\eta_n \nabla \ell(\theta_n, \xii_{n+1}),
\end{align}
where $\{\xii_n\}_{n\ge 1}$ is an i.i.d. sequence, $\mathbb{E}[\nabla \ell(\theta_n,\xii_{n+1})|\sigma(\xii_1,\cdots,\xii_n)]=\mathbb{E}_{\xii\sim P}[\nabla\ell(\theta_n,\xii)]$ and $\eta_n$ is the learning rate. It is worth noticing if $\nabla\ell(\theta,\xii)$ is replaced with a general function $H(\theta,\xii)$, the iteration \eqref{eq: sgd_theta} becomes the standard stochastic approximation procedure \cite{robbins1951stochastic} for solving the equation $\EB H(\theta,\xii)=0$. 

From a statistical inference perspective, analyzing the limiting behavior of the estimator $\theta_n$ is crucial for uncertainty quantification, such as constructing confidence intervals that support reliable decision-making. When the stochastic gradient $\nabla\ell(\theta,\xii)$ has finite variance, it can be shown \citep{sacks1958asymptotic, kushner2012stochastic} that the asymptotic distribution of $\theta_n$ satisfies:
\begin{align*}
    \sqrt{\eta_n^{-1}}\left(\theta_n-\theta^*\right)\Rightarrow \mathcal{N}(0,\Sigma),
\end{align*}
where the asymptotic variance $\Sigma$ satisfies the Lyapunov equation:
\begin{align*}
\EB[\nabla^2\ell(\theta^*,\xii)]\Sigma+\Sigma\EB[\nabla^2\ell(\theta^*,\xii)]=\Gamma,    
\end{align*}
and $\Gamma=\text{Cov}(\nabla\ell(\theta^*,\xii),\nabla\ell(\theta^*,\xii))$. It can also be equivalently shown \citep{pelletier1998weak, benaim2006dynamics} the asymptotic distribution is the stationary distribution of the following stochastic differential equation:
\begin{align*}
    d Z_t = -\EB[\nabla^2\ell(\theta^*,\xii)]Z_t dt+\Gamma^{\frac{1}{2}}dB_t,
\end{align*}
which is an Ornstein-Uhlenbeck (OU) process driven by a Brownian motion process. 

{\rv Recently, a complementary line of work has emerged on high-dimensional scaling limits of SGD and related online learning algorithms, where the dimension grows with the sample size (e.g., \cite{wang2017scaling, ben2022high, balasubramanian2023high}). These studies typically assume finite-variance noise and analyze the limiting dynamics via PDE, mean-field, or weak convergence techniques, revealing rich phase-transition phenomena governed by the step-size, dimension, and number of steps.

While these results deepen our understanding of SGD under Gaussian fluctuations, they fundamentally rely on finite-variance assumptions, prompting the question of how the asymptotic behavior changes in the presence of infinite-variance stochastic gradients.}
In practice, there are scenarios, for example in the training of deep neural networks, where the norm of the stochastic gradients grows in a manner consistent with infinite variance stochastic gradients. Throughout this paper, we will refer to infinite variance stochastic gradients as ``heavy-tailed'' noises\footnote{We caution that heavy-tailed random variables are often understood to simply those that do not possess a finite moment generating function.}. If we run SGD over $n$ iterations with stochastic gradient sizes whose growth is consistent with that of an infinite variance sequence, a reasonable model of inference for the error should be based on an infinite variance gradient model. Empirical evidence showing wide variation in stochastic gradients is reported in \citep{simsekli2019tail}. A substantial body of literature explains the emergence of heavy tails in SGD applications for training neural networks, even when the input data is light-tailed. This phenomenon arises from the multiplication of many random weights during backpropagation via the chain rule, where some of these random weights can occasionally be expansive. This random-weight phenomenon, leading to heavy tails, is well known in probability theory; see \citep{kesten1973random} for an early foundational result. Its implications in machine learning are discussed, for example, in \citep{kesten1973random, gurbuzbalaban2021heavy, hodgkinson2021multiplicative, simsekli2019tail, gurbuzbalaban2020fractional}. Furthermore, studies \citep{wang2021eliminating, csimcsekli2019heavy, simsekli2019tail, hodgkinson2021multiplicative, srinivasan2021efficient, garg2021proximal} show that heavy-tailed stochastic gradients can help SGD discover flat local minima, promoting better generalization in non-convex optimization problems. 

However, theoretical results on SGD in the presence of heavy-tailed noise remain limited. A very early reference on the asymptotic behavior of SGD in the presence of heavy-tailed noise is \citep{krasulina1969stochastic}, which focuses on the one-dimensional case, where the noise distribution lies in the domain of normal attraction of the stable law\footnote{See the discussion in Section \ref{sec: pre}.}. Subsequent studies \citep{goodsell1976almost, li1994almost} analyze the almost sure convergence of $\theta_n$. More recently, \cite{wang2021convergence} investigates the non-asymptotic convergence rate when the noise belongs to the domain of normal attraction of a stable law. However, these studies assume that the heavy-tailed noise is independent of the parameters, which is a notable limitation. We defer a detailed comparison with related works to the next paragraph. In recent research, other studies \citep{liu2023stochastic, cutkosky2021high} explore non-asymptotic convergence rates under different assumptions on the loss function, assuming that the noise distribution has a finite $\alpha$-th moment, a different class than the one considered in \cite{wang2021convergence}. Additionally, \cite{raj2023algorithmic} studies the continuous-time version of SGD with $\alpha$-stable noise and develops a discrete approximation using the Euler-Maruyama method. On the practical side, some works \cite{csimcsekli2019heavy, simsekli2019tail} investigate how SGD with heavy-tailed noise can help it escape sharp local minima. And \cite{wang2021eliminating} considers clipping the gradient to  eliminate the sharp minima from the trajectories of SGD in some suitable conditions. Despite these advances, key open questions remain: What are the asymptotic distribution and convergence rates of SGD when the noise belongs to the domain of attraction of the stable law in high-dimensional settings, especially when the noise depends on the parameters? 

\paragraph*{Related literature on asymptotic analysis of SGD with infinite variance} As mentioned above, the first limit theorems for SGD algorithm with infinite variance can source back to \cite{krasulina1969stochastic}. which is the only most closely related work within our settings. However, the results in \cite{krasulina1969stochastic} impose the following assumptions:
\begin{itemize}
    \item The stochastic noises  $\nabla \ell(\theta, \xii_{n+1})- \EB[\nabla\ell(\theta, \xii_{n+1})]$ are i.i.d. without dependence of parameter $\theta$. 
    \item The distribution of noise  belongs to the domain of normal attraction of the stable law. 
    \item The results are established only in the one-dimensional case.
\end{itemize}
However, these three assumptions are restrictive in many practical scenarios. For the first assumption, in some cases, stochastic gradient noises can depend on the parameter $\theta$. This occurs already with highly popular classification models, such as logistic regression, which we explicitly consider in Section \ref{sec: application}. For the second assumption, if the distribution of stochastic gradient noises belongs to the domain of normal attraction of the stable law, it implies the distribution has an approximately Pareto-like tail. This assumption is restrictive as the tail distribution of noises could be slightly heavier or lighter than the Pareto-like tail by a logarithmic factor. And in practice, it is challenging to test whether the noise distribution has a Pareto-like tail, as the underlying distribution of the noises may not be easily identifiable.  For the third assumption, in practical scenarios, the parameters are always high-dimensional such as neural networks \cite{simsekli2019tail}. 

Relaxing these assumptions is non-trivial. While we will explain how we relax these assumptions in the next paragraph, it is important to first discuss why this task is not straightforward. We can relax the first assumption by adopting Assumption~\ref{asmp: g_lip}, which ensures that the stochastic gradient remains bounded and does not vary significantly as the parameters change. However, generalizing the other two assumptions presents considerable challenges. On one hand, when the noise distribution belongs to the domain of attraction of a stable law, the asymptotic distribution of the final iterate becomes intractable due to the appearance of an unknown slowly varying function, making direct computation impossible. On the other hand, in the multi-dimensional case, such calculations also fail because correlations among the coordinates complicate the analysis, even if each individual noise component belongs to the domain of normal attraction of a stable law. Since 1969, no subsequent work has successfully established these results in more general settings. Even in the most recent work \cite{wang2021convergence}, the limit theorem is derived under the Polyak-averaging framework. However, their results rely on the assumption that the stochastic noises are symmetric and belong to the domain of normal attraction of the stable laws. And their results do not characterize the asymptotic distribution for the final iteration, leaving a gap in the theoretical understanding of the SGD algorithm in these more complex scenarios.

\paragraph*{Technique overview} In this paper, we derive the asymptotics for SGD algorithm when the stochastic gradient belongs to the domain of attraction of the stable law. When the learning rate $\eta_n$ is constant $\eta$, we show that the sample path of the SGD algorithm is close to the trajectory of the ground-truth gradient descent algorithm with a proper scaling rate. Specifically, in Skorokhod $J_1$ topology on $\mathbb{D}[0,\infty)$, as $\eta\to0$,
\begin{align}
\label{eq: constant}
    \eta^{\frac{1}{\alpha}-1}b_1(\eta^{-1})\left(\theta_{\frac{\cdot}{
    \eta}}-\bar{\theta}(\cdot)\right)\Rightarrow Z(\cdot),
\end{align}
where $b_1(\cdot)$ is a slowly varying function defined in Section~\ref{sec: pre}, $\bar{\theta}(\cdot)$ is the trajectory of ground-truth gradient descent algorithm and $Z(\cdot)$  is an Ornstein-Uhlenbeck process driven by an additive process.

We then consider the case that learning rate is decaying as $\eta_n=cn^{-\rho}$, where $\rho\in(0,1]$. We show that the last iterate $\theta_n$ weakly converges to the optimal solution $\theta^*$ with a proper scaling rate, as $n\to\infty$:
\begin{align}
\label{eq: decay}
    \eta_n^{\frac{1}{\alpha}-1}b_1(\eta_n^{-1})\left(\theta_n-\theta^*\right)\Rightarrow Z_{\infty},
\end{align}
where $Z_{\infty}$ is a stationary distribution of a L\'evy-driven {\rv Ornstein-Uhlenbeck} process.

To prove \eqref{eq: constant}, we adapt the classic technique from \cite{resnick2007heavy}, where they establish a functional limit theorem for partial sums of i.i.d. regularly varying random variables. For \eqref{eq: decay}, we build on ideas from \cite{pelletier1998weak, kushner2012stochastic, harold1997stochastic}, constructing a continuous-time process from the discrete recursion of the SGD iteration \eqref{eq: sgd_theta} and analyzing the weak convergence of this new process. The results for both constant and decaying learning rates exhibit structural similarities to those obtained when the gradient noise has finite variance. In the finite-variance case, with a constant learning rate, the limiting process is an inhomogeneous Ornstein-Uhlenbeck process driven by Brownian motion \cite{fierro2008stochastic}. For a decaying learning rate, the limiting behavior corresponds to the stationary distribution of an Ornstein-Uhlenbeck process driven by Brownian motion \cite{sacks1958asymptotic, kushner2012stochastic}. These finite-variance results are relatively straightforward to derive, as the existence of finite second-order moments allows effective control of residual terms through appropriate scaling. In contrast, the infinite-variance case presents significant challenges, as we generally lack a finite $\alpha$-th order moment. This absence renders many techniques from the finite-variance case inapplicable. {\rv  A straightforward example is the typical use of Markov's inequality cannot yield the correct asymptotic convergence rate of order $1-\frac{1}{\alpha}$. In the meantime}, without a finite $\alpha$-moment condition, establishing the tightness of the scaled random sequence in \eqref{eq: decay} {\rv is non-trivial, which is a key step in proving weak convergence}. To address these difficulties, we first prove the convergence result of \eqref{eq: decay} under symmetric noise conditions. For the asymmetric noise case, we decompose the noise into two components, each of which can be symmetrized independently. This approach allows us to leverage the intermediate convergence result obtained for symmetric noise, enabling a complete analysis of the asymmetric case.

The paper is organized as follows. In the end of Section~\ref{sec: introduction}, we introduce some notations. In Section~\ref{sec: pre}, we introduce some basic concepts and assumptions in this paper. In Section~\ref{sec: constant}, we state the main results for constant learning rate case. In Section~\ref{sec: decay}, we state the main results for decaying learning rate case. In Section~\ref{sec: application}, we discuss the application of our results to classic machine learning models. {\rv Section~\ref{sec: exp} contains the numerical results. }And Section~\ref{sec: proof} contains all the proofs. {\rv We make a conclusion in Section~\ref{sec: conclusion}.}

\paragraph*{Notations}$\bar{\RB}^d:=\RB^d\cup\{\pm\infty\}$ and $\bar{\RB}^d_\varepsilon:=\{x\in\bar{\RB}^d|\|x\|>\varepsilon\}$. $\boldsymbol{M}_p$ is the space of radon point measure. $\|\cdot\|$ is the standard $\ell_2$ norm in $\RB^d$. A stochastic process $\{X_{t}\}_{t\in[0,T]}$ is called a $\cadlag$ process if for every $t\in[0,T]$, $X_t$ is right continuous with left limits almost surely. For a sequence of stochastic process $\{X_t^n\}_{t\in[0,T]}$, we denote $X^n(\cdot)\Rightarrow X(\cdot)$ for $X^n(\cdot)$ weakly converging to $X(\cdot)$ in Skorokhod $J_1$ topology. For a sequence of measure $\mu_n$, we denote $\lim_{n\to\infty}\mu_n\overset{v}{=}\mu$ and $\mu_n\overset{v}{\to}\mu$ for $\mu_n$ vaguely converges to $\mu$. We also denote $f(x)\sim g(x)$ ($x\to\infty$) for $\lim_{x\to\infty}\frac{f(x)}{g(x)}=C$, where $C$ is a non-zero constant. {\rv A $(d-1)$-dimensional sphere is denoted as $\mathbb{S}^{d-1}:=\left\{x\in\RB^{d}|\|x\|=1\right\}$.}

\section{Preliminaries}
\label{sec: pre}
In this section, we will introduce some basic concepts and assumptions occurred in this paper. First, we introduce the definition of additive processes in the following.

\begin{defn}[Additive process]
\label{def: add}
    A stochastic process $X_t$ is called an additive process if it satisfies:
    \begin{itemize}
        \item It has independent increments, i.e. for any $0\le p<r\le s<t$, $X_t-X_s$ is independent with $X_r-X_p$.
        \item It is continuous in probability, i.e. for any $t>0$ and $\varepsilon>0$, $\lim_{s\to 0^+}\PB(|X_{t+s}-X_t|\ge\varepsilon)=0$.
    \end{itemize}
\end{defn}
\begin{defn}[Characteristics of additive process]
\label{def: add_char}
    An additive process $X_t$ has characteristics $(A_t, \gamma_t,\nu_t(\cdot))$ if its characteristic function satisfies:
    \begin{align*}
        \EB[\exp(iu^\top X_t)]=\exp\left(u^\top\gamma_t i-\frac{1}{2}u^\top A_t u+\int\left(e^{iu^\top x}-1-iu^\top x\1(\|x\|<1)\right)\nu_t(dx)\right).
    \end{align*}
    Moreover, if $A_t=A t$, $\gamma_t = \gamma t$ and $\nu_t(\cdot)=\nu(\cdot)t$, which means $X_t$ has a stationary increment, we call $X_t$ is a L\'evy process with characteristics {\rv $(A,\nu(\cdot),\gamma)$}.
\end{defn}
In this paper, we always assume the stochastic process is a $\cadlag$ process as an additive process has a $\cadlag$ modification. It is worth noticing that a Brownian motion process is a special additive process. Besides, an additive process may have countable jump points on its trajectory, such as Poisson process. In Definition~\ref{def: add_char}, there are three parameters that determine the characteristic function of $X_t$. The $A_t$ is the covariance matrix, $\gamma_t$ is the drift term, and $\nu_t(\cdot)$ is the L\'evy measure. {\rv In the following, we introduce some basic concepts in infinite variance analysis. The Definitions~\ref{def: reg_varying} and \ref{def: domain} are both in 1-dimension.}

\begin{defn}[Regular varying]
\label{def: reg_varying}
We say a function $f(\cdot)$ is $\alpha$-regular varying if it satisfies for any $x\in\RB_+$:
\begin{align*}
    \lim_{t\to+\infty}\frac{f(tx)}{f(t)}=x^{\alpha},
\end{align*}
where $\alpha\in\RB$. When $\alpha=0$, we call $f(\cdot)$ is slowly varying. We also call a random variable $X$ is $\alpha$-regular varying if the function $\PB(|X|>t)$ is $(-\alpha)$-regular varying.
\end{defn}

\begin{defn}[Domain of attraction]
\label{def: domain}
    For a distribution $F(\cdot)$, we say it belongs to a domain of attraction of the stable laws with exponent $\alpha\in(0,2)$ as long as there exists a sequence $A_n$ and $B_n$ such that the following holds:
    \begin{align*}
        \frac{\sum_{i=1}^n X_i}{B_n}-A_n\Rightarrow Z,
    \end{align*}
    where $\{X_i\}_{i=1}^n$ are the i.i.d. random variables following distribution $F(\cdot)$ and $Z$ is an $\alpha$-stable random variable. Moreover, if $B_n\sim n^{\frac{1}{\alpha}}$, we say $F(\cdot)$ belongs to a domain of normal attraction of the stable laws.
\end{defn}
In Definition~\ref{def: domain}, if the distribution function $F(\cdot)$ satisfies:
\begin{align*}
    1-F(x)+F(-x)=x^{-\alpha}b(x),
\end{align*}
where $x>0$ and $b(x)$ is a slowly varying function, then it belongs to a domain of attraction of the stable laws. Moreover, if $b(x)=c+o(1)$  as $x\to+\infty$, it can be proved it belongs to the domain of normal attraction of the stable laws \cite{gnedenko1968limit}. In the following, we introduce the core assumptions for the SGD algorithm. The following assumption is indeed a multi-variate version of {\rv 
regular varying}. {\rv The motivation for the following assumption is due to the empirical evidence showing that the stochastic gradient follows a power-law without finite variance in \cite{simsekli2019tail}. To facilitate the limit theorem analysis, the following assumption is required, which is also standard in heavy-tail analysis in \cite{resnick2007heavy}.}

\begin{assumption}[Multi-dimensional regular varying]
    \label{asmp: heavy_tail}
    There exists an $\alpha\in(1,2)$, a bounded Lipschitz function $C:\RB^d\to\RB_+$, and a differentiable slowly varying function $b_0(\cdot,\theta):\RB_+\to\RB_+$, such that the norm of stochastic gradient 
    $\nabla \ell(\theta,\xii)$ satisfies:
    \begin{align}
    \label{eq: tail_norm}
        \PB\left(\left\|\nabla \ell(\theta,\xii)\right\|>x\right)=x^{-\alpha}b_0(x,\theta),
    \end{align}
    where the slowly varying function $b_0(\cdot,\theta)$ also satisfies $\lim_{x\to\infty}\frac{b_0(x,\theta)}{b_0(x,\theta^*)}=C(\theta)$ and $\lim_{x\to\infty}\frac{xb_0'(x,\theta)}{b_0(x,\theta)}=0$. And there exists a Radon measure $\nu(\theta,\cdot)$ on $\RB^d\setminus\{0\}$ and a slowly varying function $b_1:\RB_+\to\RB_+$, such that:
    \begin{align*}
        x\PB\left(\frac{b_1(x)\nabla\ell(\theta,\xii)}{x^{\frac{1}{\alpha}}}\in\cdot\right)\overset{v}{\to}\nu(\theta,\cdot),\text{ $x\to+\infty$},
    \end{align*}
    where $\nu(\theta,\{\|z\|>x\})<\infty$ for any $x>0$ and $\int_{\|z\|\in(0,1]}\|z\|^2\nu(dz)<\infty$. In fact, $\nu(\theta,\cdot)$ can also be decomposed into $\nu(\theta,dz)=\alpha \widetilde{C}(\theta)\mu(\theta,d\omega)r^{-\alpha-1}dr$ with $z=r\cdot\omega$ being the  polar coordinate decomposition, where $\mu(\theta,d\omega)$ is a probability measure on unit sphere $\mathbb{S}^{d-1}$. Moreover, the vague convergence is uniformly in $\theta$ belonging to a compact set. 
\end{assumption}
In Assumption~\ref{asmp: heavy_tail}, we require the slowly varying function $b_0(x,\theta)$ to be differentiable and also $\lim_{x\to\infty}\frac{x b_0'(x,\theta)}{b_0(x,\theta)}=0$, which makes contribution to the existence of $b_1(x)$ and uniform integrability of the measure $x\PB\left(\frac{b_1(x)\nabla\ell(\theta,\xii)}{x^{\frac{1}{\alpha}}}\in\cdot\right)$ when $x$ is large. In fact, this requirement can be relaxed as for any slowly varying function $b_0(x,\theta)$, there exists a smooth regular varying function $\widetilde{b}_0(x,\theta)$ such that $b_0(x,\theta)\sim \widetilde{b}_0(x,\theta)$ as $x\to\infty$ and $\lim_{x\to\infty}\frac{x \widetilde{b}_0'(x,\theta)}{\widetilde{b}_0(x,\theta)}=0$. Then, we can construct the $b_1(x)$ via $\widetilde{b}_0(x,\theta)$. See Theorem 1.3.3 in \cite{Bingham_Goldie_Teugels_1987}. In this paper, we choose $b_1(x)$ such that $\PB(\|\nabla\ell(\theta^*,\xii)\|>x^{\frac{1}{\alpha}}b_1^{-1}(x))=x^{-1}$ while it takes no difference by choosing any $\widetilde{b}_1(x)=b_1(x)+o(1)$ when $x\to\infty$. {\rv Or equivalently, $b_1(x)^{\alpha}b_0\left(x^{\frac{1}{\alpha}}b_1^{-1}(x),\theta^*\right)=1$.}

It has recently been established that the presence of a non-trivial slowly varying function significantly reduces the rate of convergence in the Wasserstein metric in the stable domain of attraction, see~\cite{cázares2024asymptoticallyoptimalwassersteincouplings}.  It is thus important to allow in our results the function $b_0(x,\theta)$ not to be asymptotically equivalent to a positive constant as $x\to\infty$. {\rv In the following, we introduce the basic assumptions for the loss function.}

{\rv
\begin{assumption}[Assumption 1 in \cite{wang2021convergence}]
    \label{asmp: ppd}
    The set $\{\nabla^2\ell(\theta):\theta\in\RB^d\}$ is bounded and uniformly p-positive definite for any $p\in(1,\alpha)$. Equivalently, for any $\theta\in\RB^d$ and $\|u\|_p=1$, $\nabla^2\ell(\theta)$ is bounded and satisfies $u^\top\nabla^2\ell(\theta)\left(\text{sgn}(u_1)|u_1|^{p-1},\cdots,\text{sgn}(u_d)|u_d|^{p-1}\right)^\top>0$ for any $p\in(1,\alpha)$.
\end{assumption}
}

\begin{assumption}
    \label{asmp: g_lip}
    There exists $C_{\text{Lip}}>0$ such that for any $\theta,\theta'$, we have:
    \begin{align*}
        \|\nabla\ell(\theta,\xii)-\nabla\ell(\theta',\xii)\|\le C_{\text{Lip}}\|\theta-\theta'\|, \text{ a.s.}.
    \end{align*}
\end{assumption}

\begin{assumption}
    \label{asmp: q_expansion}
    There exists $q\in(2-\frac{1}{\alpha},\alpha)$ and constant $K>0$, such that:
    \begin{align*}
        \left\|\nabla\ell(\theta_1)-\nabla\ell(\theta_2)-\nabla^2\ell(\theta_2)(\theta_1-\theta_2)\right\|\le K\|\theta_1-\theta_2\|^{q}.
    \end{align*}
\end{assumption}
{\rv Assumption~\ref{asmp: ppd} is first proposed in \cite{wang2021convergence} to obtain the finite-time convergence rate for SGD iterates in multi-dimensional case. In general, Assumption~\ref{asmp: ppd} is a stronger assumption on the strongly convexity of the loss function $\ell(\theta)$. Assumption~\ref{asmp: g_lip} is a traditional assumption on the smoothness of the stochastic gradient.} Assumption~\ref{asmp: q_expansion} is a standard assumption, which is also stated in \cite{wang2021convergence, polyak1992acceleration}. This assumption makes the nonlinear approximation residual term converge to zero almost surely.

\section{Constant step size}
\label{sec: constant}
In this section, we consider the stochastic approximation of SGD to the GD solution path with a small constant stepsize $\eta$. Mathematically, we denote the gradient flow as:
\begin{align*}
    d\bar{\theta}(t)=-\nabla \ell(\bar{\theta}(t))dt,
\end{align*}
while the stochastic gradient flow as:
\begin{align*}
    \theta_{\eta}((n+1)\eta)=\theta_{\eta}(n\eta)-\eta \nabla \ell(\theta_{\eta}(n\eta),\xii_{(n+1)\eta}).
\end{align*}
Instead of studying the limit of $\theta_{\eta}(n\eta)-\theta^*$ as $n\to\infty$ and $\eta\to0$, we consider a more general functional limiting behavior of $\theta_{\eta}(\cdot)-\bar{\theta}(\cdot)$ on every $[0,T]$, which measures the stochastic deviation in the trajectory.
\begin{theorem}
    \label{thm: constant}
    Under Assumption~\ref{asmp: heavy_tail}, \ref{asmp: g_lip} and \ref{asmp: q_expansion}, {\rv if $\nabla^2\ell(\bar{\theta}(t))$ is continuous}, in Skorokhod $J_1$ topology on $\mathbb{D}[0,\infty)$, as $\eta\to0$,
    \begin{align}
    \label{eq: const_error}
        \eta^{\frac{1}{\alpha}-1}b_1(\eta^{-1})\left(\theta_{\eta}(\cdot)-\bar{\theta}(\cdot)\right)\Rightarrow {\rv Z_{\cdot}},
    \end{align}
where $\rv Z_{\cdot}$ satisfies the s.d.e.  $dZ_t = -\nabla^2\ell(\bar{\theta}(t))Z_t dt+dL_t$. And $L_{\cdot}$ is an additive process with characteristics $(0,\nu_t(\cdot),-\gamma_t)$:
\begin{align*}
    &\nu_t(dx) = \int_{[0,t]}\nu(\bar{\theta}(s), dx)ds,\\
    &\gamma_t = \int_{s\in[0,t],\|x\|>1}x\nu(\bar{\theta}(s),dx)ds.
\end{align*}
Or equivalently,
\begin{align*}
    Z_t = \Phi(t)\left(Z_0+\int_0^t\Phi(s)^{-1}dL_s\right),
\end{align*}
and $\Phi(\cdot)$ is the solution of the following o.d.e.:
    \begin{align*}
        \begin{cases}
            &d\Phi(t)=-\nabla^2 \ell(\bar{\theta}(t))\Phi(t)dt,\\
            &\Phi(0)=I_d.
        \end{cases}
    \end{align*}
\end{theorem}
{\rv In Theorem~\ref{thm: constant}, the continuity assumption for $\nabla^2\ell(\bar{\theta}(t))$ is required for the unique existence of $\Phi(\cdot)$ by \cite{baake2011peano}.} Theorem~\ref{thm: constant} characterizes the limit properties of the scaled parameter errors over the trajectory, which follow an OU-process driven by an additive process. {\rv It is also worth noticing that strongly convexity is not required in Theorem~\ref{thm: constant} as the limit theorem holds along the ground-true path instead of optimal single point. }Comparing with the results in finite variance case in \cite{fierro2008stochastic}, the difference is that the Brownian motion term in \cite{fierro2008stochastic} is replaced with the additive process $L_t$, which is due to the assumption of infinite variance noise. It is worth noticing that the stochastic process $\rv L_{\cdot}$  in Theorem~\ref{thm: constant} is time-inhomogeneus, which is due to the centering parameter $\bar{\theta}(\cdot)$ changes w.r.t. time. In Theorem~\ref{thm: constant}, we also notice that $\rv Z_{\cdot}$ is determined by a stochastic integral w.r.t. an additive process $\rv L_{\cdot}$. By Corollary 2.19 in \cite{sato2006additive}, we can also prove $Z_t$ is an additive process as follows. It is also worth noticing that $\widetilde{\gamma}_t$ measures the skewness of the distribution of $Z_t$. If $\widetilde{\nu}_t(\cdot)$ is symmetric, it is straightforward that $\widetilde{\gamma}_t=0$ and $Z_t$ is thus symmetric.
\begin{proposition}
\label{cor: charac_Z}
    In Theorem~\ref{thm: constant}, $\rv Z_{\cdot}$ is also an additive process with characteristics $(0,\widetilde{\nu}_t(\cdot),\widetilde{\gamma}_t)$:
    \begin{align*}
        \widetilde{\gamma}_t =& -\int_{s\in[0,t],\|x\|>1}\Phi(t)\Phi(s)^{-1}x\nu(\bar{\theta}(s),dx)ds\\
        &+\int_{s\in[0,t]}\Phi(t)\Phi(s)^{-1}x\left(\1(\|\Phi(t)\Phi(s)^{-1}x\|\le 1)-\1(|x|\le1)\right)\nu(\bar{\theta}(s),dx)ds,\notag\\
        \widetilde{\nu}_t(A)=&\int_{s\in[0,t],\Phi(t)\Phi(s)^{-1}x\in A}\nu(\bar{\theta}(s), dx)ds.
    \end{align*}
\end{proposition}

\section{Decaying step size}
\label{sec: decay}
In this section, we consider the stochastic approximation of SGD to the optimal solution with decaying stepsize. We denote the optimal solution as $\theta^*\in\argmin_{\theta} \EB_{\xi\sim P}\ell(\theta,\xii)$ and the stochastic gradient flow as:
\begin{align}
\label{eq: sgd}
    \theta_{n+1}=\theta_{n}-\eta_n \nabla\ell(\theta_{n},\xii_{n+1}).
\end{align}
In this section, we alter to study the limit behavior of error $\theta_n-\theta^*$ instead of the trajectory error, as the gradient descent with the decaying learning rate will find the optimal solution with linear convergence rate. In the following, we are ready to present the formal result.

\begin{theorem}
    \label{thm: asymp_decay}
    Under {\rv Assumption~\ref{asmp: heavy_tail}, \ref{asmp: ppd}, \ref{asmp: g_lip} and \ref{asmp: q_expansion}}, if $\eta_n = c\cdot n^{-\rho}$ ($\rho\in(0,1)$), we have:
    \begin{align}
    \label{eq: decay_error}
        \eta_n^{\frac{1}{\alpha}-1}b_1(\eta_n^{-1})\left(\theta_n-\theta^*\right)\Rightarrow Z_{\infty},
    \end{align}
    where $Z_{\infty}$ is the stationary distribution of the following s.d.e.:
    \begin{align}
        \label{eqn: decay_sde}
        dZ_t = -\nabla^2 \ell(\theta^*) Z_t dt+dL_t,
    \end{align}
    where $L_{\cdot}$ is a L\'evy process with characteristics $(0,\nu(\theta^*,\cdot),-\gamma)$ and:
    \begin{align*}
        \gamma=\int_{\|x\|>1}x\nu(\theta^*,dx).
    \end{align*}
\end{theorem}
Theorem~\ref{thm: asymp_decay} characterizes the asymptotic distribution of the scaled final iteration error, which turns out to be the stationary distribution of an OU-process driven by a L\'evy process. Comparing with the results in finite variance case in \cite{pelletier1998weak}, the difference is that the diffusion term is replaced from a Brownian motion with a L\'evy process. Moreover, with a careful time-reverse transformation, we can show the stationary distribution of the s.d.e.~\eqref{eqn: decay_sde} satisfies the following proposition.
\begin{proposition}
\label{prop: z_stat_prop}
    In Theorem~\ref{thm: asymp_decay}, the solution $Z_{\infty}$ admits an explicit solution:
    \begin{align}
    \label{eqn: z_stationary}
        Z_{\infty}=\int_{0}^{\infty}\exp(-\nabla^2 \ell(\theta^*)t)dL_t.
    \end{align}
    And its characteristic function satisfies:
    \begin{align*}
        \EB\exp(iu^\top Z_{\infty})=\exp\left(u^\top\widetilde{\gamma}i+\int_{\RB^d\setminus\{0\}}\left(e^{iu^\top x}-1-iu^\top x\1(\|x\|<1)\right)\widetilde{\nu}(dx)\right),
    \end{align*}
    where
    \begin{align*}
        &\widetilde{\gamma}=-\nabla^2\ell(\theta^*)^{-1}\int_{\|x\|>1}x\nu(\theta^*,dx)\\
        &+\int_{t\ge0}\exp(-\nabla^2\ell(\theta^*)t)x\left(\1(\|\exp(-\nabla^2\ell(\theta^*)t)x\|\le 1)-\1(\|x\|\le 1)\right)\nu(\theta^*,dx)dt,\notag\\
        &\widetilde{\nu}(A)=\int_{t\ge0, \exp(-\nabla^2\ell(\theta^*)t)x\in A}\nu(\theta^*,dx)dt.
    \end{align*}
\end{proposition}
Moreover, in Theorem~\ref{thm: asymp_decay}, we assume the learning rate decays with rate slower than linear stepsize $n^{-1}$. In the following, we show a similar result when the learning rate is chosen by the linear stepsize.
\begin{corollary}
\label{thm: asymp_decay_lr}
    Under Assumption~\ref{asmp: heavy_tail}, \ref{asmp: g_lip} and \ref{asmp: q_expansion}, if $\eta_n = c\cdot n^{-1}$, where $c>\frac{1-\frac{1}{\alpha}}{\sigma_{\min}(\nabla^2\ell(\theta^*))}$, we have:
    \begin{align*}
        \eta_n^{\frac{1}{\alpha}-1}b_1(\eta_n^{-1})\left(\theta_n-\theta^*\right)\Rightarrow Z_{\infty},
    \end{align*}
    where $Z_{\infty}$ is the stationary distribution of the following s.d.e.:
    \begin{align}
        \label{eqn: decay_sde_lr}
        dZ_t = \left(-\nabla^2 \ell(\theta^*)+\frac{1-\frac{1}{\alpha}}{c}I\right) Z_t dt+dL_t,
    \end{align}
    where $L_{\cdot}$ is a L\'evy process with characteristics $(0,\nu(\theta^*,\cdot),-\gamma)$ and:
    \begin{align*}
        \gamma=\int_{\|x\|>1}x\nu(\theta^*,dx).
    \end{align*}
\end{corollary}
It is worth noticing that in Eqn~\eqref{eqn: decay_sde_lr}, there is an additional constant in the drift term comparing with Eqn~\eqref{eqn: decay_sde}, which is due to the ratio of scaling rate  is heavier:
\begin{align*}
    \left(\frac{\eta_{n+1}}{\eta_n}\right)^{\frac{1}{\alpha}-1}=
    \begin{cases}
        1+\frac{1-\frac{1}{\alpha}}{c}\eta_n+o(\eta_n),&\text{ when $\eta_n=c\cdot n^{-1}$,}\\
        1+o(\eta_n),&\text{ when $\eta_n=c\cdot n^{-\rho}$.}
    \end{cases}
\end{align*}
Thus, when extracting the error from $\eta_n^{\frac{1}{\alpha}-1}b_1(\eta_n^{-1})(\theta_n-\theta^*)$, the additional term should not be neglected for the case $\eta_n=c\cdot n^{-1}$. Moreover, to guarantee the existence of the stationary distribution $Z_{\infty}$, we also require $c>\frac{1-\frac{1}{\alpha}}{\sigma_{\min}(\nabla^2\ell(\theta^*))}$.

\section{Applications}
\label{sec: application}
In this section, we introduce the classic machine learning models: linear regression model and logistic regression model to see how the heavy tail behavior arises (Assumption~\ref{asmp: heavy_tail}) in each setting.
\subsection{Ordinary Least Squares}
We consider a linear model:
\begin{align*}
    y_i=x_i^\top\theta^*+\varepsilon_i,
\end{align*}
where $\varepsilon_i$ is a mean zero random variables and satisfy:
\begin{align}
\label{eq: eps_pareto}
    \PB\left(|\varepsilon_i|>t\right)=\frac{\1(t>1)}{t^\alpha}+\1(0\le t\le1),
\end{align}
and $x_i$ is a sub-Gaussian random variables. Moreover, it also satisfies $\mathbb{E}[x_ix_i^\top]$ is positive definite and $x_i$ and $\varepsilon_i$ are independent. And $\theta^*\in\RB^d$ is the true parameter. The stochastic loss function $\ell$ is defined as:
\begin{align*}
    \ell(\theta,x_i,\varepsilon_i)=\frac{1}{2}\left(y_i-x_i^\top\theta\right)^2.
\end{align*}
And the stochastic gradient $\nabla_{\theta}\ell$ is:
\begin{align}
\label{eq: ols_grad}
    \nabla\ell(\theta,x_i,\varepsilon_i)=x_ix_i^\top(\theta-\theta^*)-x_i\varepsilon_i.
\end{align}
Under the heavy tail assumption of $\varepsilon_i$ in \eqref{eq: eps_pareto}, we can show that the stochastic gradient $\nabla\ell(\theta,x_i,\varepsilon_i)$ is an $\alpha$-regular varying random variable in the following.
\begin{theorem}
\label{thm: ols_norm}
    If $x_i$ is generated from a compact set and $\varepsilon_i$ satisfies~\eqref{eq: eps_pareto}, the gradient $\nabla\ell(\theta,x_i,\varepsilon_i)$ in~\eqref{eq: ols_grad} is $\alpha$-regular varying as:
    \begin{align*}
    &\lim_{z\to+\infty}\frac{\PB(\|\nabla\ell(\theta,x_i,\varepsilon_i)\|>tz)}{\PB(\|\nabla\ell(\theta,x_i,\varepsilon_i)\|>z)}=t^{-\alpha},\\
    &\lim_{z\to+\infty}\PB\left(\left.\frac{\nabla\ell(\theta,x_i,\varepsilon_i)}{\|\nabla\ell(\theta,x_i,\varepsilon_i)\|}\in\cdot\right|\|\nabla\ell(\theta,x_i,\varepsilon_i)\|>z\right)\overset{v}{=}\frac{\EB\1_{\cdot}(\frac{x_i}{\|x_i\|})\|x_i\|^\alpha+\EB\1_{\cdot}(\frac{-x_i}{\|x_i\|})\|x_i\|^\alpha}{2\EB\|x_i\|^{\alpha}}.\notag
    \end{align*}
\end{theorem}
{\rv Applying Theorem~\ref{thm: ols_norm} to Proposition~\ref{prop: z_stat_prop}, the characteristics of limit distribution for OLS can be obtained.}

\subsection{Logistic Regression}
\label{subsec: lr}
In this part, we consider the binary logistic regression model:
\begin{align}
\label{eq: logistic_1}
    &\PB\left(y_i=1|x_i,\theta^*\right)=\frac{1}{1+\exp(-x_i^\top\theta^*)}\\
\label{eq: logistic_2}    
    &\PB\left(y_i=-1|x_i,\theta^*\right)=\frac{\exp(-x_i^\top\theta^*)}{1+\exp(-x_i^\top\theta^*)},
\end{align}
where $y_i\in\{-1,1\}$ is the binary response, $x_i$ is an $\alpha$-regular varying random vector, and $\theta^*\in\RB^d$ is the true parameter. The loss function is:
\begin{align*}
    \ell(\theta,x,y)=\ln\left(1+\exp\left(-y\theta^\top x\right)\right)+\frac{\lambda}{2}\|\theta\|^2,
\end{align*}
where $\lambda>0$ is the regularized coefficient to guarantee the loss function is strongly convex. The stochastic gradient of $\ell(\theta,x_i,y_i)$ is:
\begin{align}
\label{eq: logistic_grad}
    \nabla \ell(\theta,x_i,y_i)=\frac{-y_i\exp(-y_i\theta^\top x_i)}{1+\exp(-y_i\theta^\top x_i)}x_i+\lambda\theta.
\end{align}
Under the heavy tail assumption of $x_i$, we can show that the stochastic gradient $\nabla\ell(\theta,x_i,y_i)$ is also an $\alpha$-regular varying random vector in the following.
\begin{theorem}
\label{thm: logistic_measure}
    If the random vector $x_i$ is $\alpha$-regular varying as:
    \begin{align*}
        \lim_{z\to+\infty}\frac{\PB(\|x_i\|>tz)}{\PB(\|x_i\|>z)}=t^{-\alpha},
        \lim_{z\to+\infty}\PB\left(\left.\frac{x_i}{\|x_i\|}\in\cdot\right|\|x_i\|>z\right)\overset{v}{=}\mu(\cdot),
    \end{align*}
    where there exists $\beta_1,\beta_2>0$ such that $\int_{\omega^\top\theta^*>0}(\omega^\top\theta^*)^{-\beta_1}\mu(d\omega)<+\infty$ and $\int_{\omega^\top\theta^*<0}(-\omega^\top\theta^*)^{-\beta_2}\mu(d\omega)<+\infty$, and $y_i$ satisfies~\eqref{eq: logistic_1} and \eqref{eq: logistic_2}. Then, the norm of gradient in \eqref{eq: logistic_grad} is regular varying:
    \begin{align*}
        \lim_{z\to+\infty}\frac{\PB(\|\nabla\ell(\theta,x_i,y_i)\|>tz)}{\PB(\|\nabla\ell(\theta,x_i,y_i)\|>z)}=t^{-\alpha}.
    \end{align*}
    We also denote the limit measure of the direction of gradient in \eqref{eq: logistic_grad} as:
    \begin{align*}
        {\rv\mu^{\dagger}(\theta,\cdot)}=\lim_{z\to+\infty}\PB\left(\left.\frac{\nabla\ell(\theta,x_i,y_i)}{\|\nabla\ell(\theta,x_i,y_i)\|}\in\cdot\right|\|\nabla\ell(\theta,x_i,y_i)\|>z\right)
    \end{align*}
    If $\theta\not=0, \theta^*\not=0$ and $\theta\not=\theta^*$:
    \begin{align*}
        {\rv\mu^{\dagger}(\theta,\cdot)}=\frac{\mu(\cdot\cap\{x\in\mathbb{S}^{d-1}|x^\top\theta>0,x^\top\theta^*<0\})+\mu((-\cdot)\cap\{x\in\mathbb{S}^{d-1}|x^\top\theta<0,x^\top\theta^*>0\})}{\mu(\{x\in\mathbb{S}^{d-1}|x^\top\theta<0,x^\top\theta^*>0\})+\mu(\{x\in\mathbb{S}^{d-1}|x^\top\theta>0,x^\top\theta^*<0\})}.\notag
    \end{align*}
    If $\theta=0$ and $\theta^*\not=0$:
    \begin{align*}
        {\rv\mu^{\dagger}(\theta,\cdot)}=\mu(\cdot\cap\{x\in\mathbb{S}^{d-1}|x^\top\theta^*<0\})+\mu((-\cdot)\cap\{x\in\mathbb{S}^{d-1}|x^\top\theta^*>0\}).\notag
    \end{align*}
    If $\theta\not=0$ and $\theta^*=0$:
    \begin{align*}
        {\rv\mu^{\dagger}(\theta,\cdot)}=\mu(\cdot\cap\{x\in\mathbb{S}^{d-1}|x^\top\theta>0\})+\mu((-\cdot)\cap\{x\in\mathbb{S}^{d-1}|x^\top\theta<0\}).\notag
    \end{align*}
    Specifically, if $\theta=\theta^*\not=0$, the support of direction measure ${\rv\mu^{\dagger}(\theta,\cdot)}$ will degenerate to $\mathbb{S}^{d-2}$. For example, if $d=3$, the support of ${\rv\mu^{\dagger}(\theta,\cdot)}$ degenerate to a {\rv 1-dimensional sphere/circle}. If $d=2$, the support of ${\rv\mu^{\dagger}(\theta,\cdot)}$ will degenerate to a two-point set. If $d=1$, the support of ${\rv\mu^{\dagger}(\theta,\cdot)}$ will degenerate to a single-point set {\rv or empty set}. If $\theta=\theta^*=0$, ${\rv\mu^{\dagger}(\theta,\cdot)}=\frac{1}{2}\mu(\cdot)+\frac{1}{2}\mu(-\cdot)$.
\end{theorem}
{\rv Applying Theorem~\ref{thm: logistic_measure} to Proposition~\ref{prop: z_stat_prop}, the characteristics of limit distribution for logistic regression can be obtained.} It is worth noticing that the limit measure of the direction of gradient in~\eqref{eq: logistic_grad} depends not only on the parameter of the true model $\theta^*$ but also depends on the current parameter $\theta$, which is different from the ordinary least squares' case. {\rv Moreover, in Theorem~\ref{thm: logistic_measure} with $d=1$, the support $\mu^\dagger(\cdot)$ may even degenerate to an empty set, which implies the stochastic gradient will be no longer heavy-tailed. 
}
{\rv
\begin{corollary}
\label{cor: log_1d}
    In Theorem~\ref{thm: logistic_measure}, when $d=1$ and $\theta^*>0$ (the result is similar when $\theta^*<0$), the limit angular measure $\mu^\dagger(\cdot)$ satisfies:
    \begin{align*}
        \mu^\dagger(\theta,\cdot)=
        \begin{cases}
            \1_{\{-1\}}\left(\cdot\right),&\text{ if $\theta\le 0$}\\
            0,&\text{ if $\theta>0$}
        \end{cases}
    \end{align*}
\end{corollary}
The degeneration of $\mu^\dagger(\theta,\cdot)$ is directly due to there is no intersection area between $\{|x|=1|x^\top\theta<0\}$ and $\{|x|=1|x^\top\theta^*>0\}$, which implies the result in Theorem~\ref{thm: logistic_measure} is not applicable in 1-dimensional case with $\theta>0$ and $\theta^*>0$. Equivalently, in this case, $\lim_{z\to\infty}\PB(y_ix_i>0||x_i|>z)=1$. Then, the stochastic gradient satisfies $|\nabla\ell(\theta,x_i,y_i)|\lesssim |x_i|\exp(-|x_i|)$ when $|x_i|$ is large, which implies the stochastic gradient is bounded. For Theorem~\ref{thm: constant}, if $\bar{\theta}(\cdot)>0$ during time $[t_1,t_2]$, the correct scaling rate for the limit theorem is $\eta^{-\frac{1}{2}}$ during time $[t_1,t_2]$ and the limit process is driven by Brownian motion. Similarly, the correct scaling rate for Theorem~\ref{thm: asymp_decay} is $\eta_n^{-\frac{1}{2}}$ and the limit distribution is Gaussian-type in this case.
}

\section{Numerical Simulation}
\label{sec: exp}
\begin{figure}[t!]
    \centering
    \includegraphics[width=1.0\linewidth]{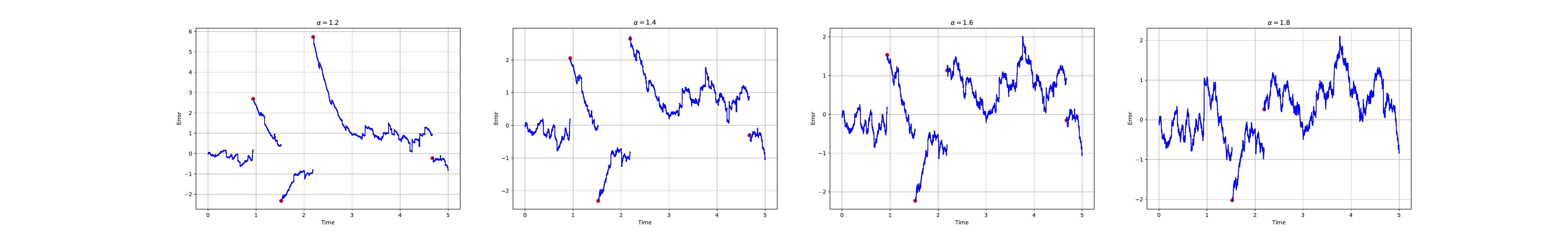}
    \caption{Ordinary Least Squares with constant learning rate $\eta=10^{-3}$ and iterations $n=5\times 10^3$.  The x-axis represents time, defined as the product of the number of iterations and the learning rate. The y-axis represents the scaled error, as defined in \eqref{eq: const_error}. Title $\alpha$ indicates the noise is generated from an $\alpha$-stable distribution. Red points mark ``jump points'', identified where the difference in scaled error exceeds a threshold of 1.}
    \label{fig: constant_alpha}
\end{figure}
In this section, we present small-scale simulations to approximate the limit theorems in this paper. The experiment settings are designed via the proposed models in Section~\ref{sec: application}.
\subsection{Ordinary Least Squares}
For the ordinary least squares problem, we consider solving the following problem:
\begin{align*}
    \min_{\theta\in \RB^2}\frac{1}{2}\theta^\top A\theta+\theta^\top b, 
\end{align*}
where $A=\diag\{2,1\}$ and $b=(1,1)^\top$. The SGD iteration is:
\begin{align*}
    \theta_{i+1}=\theta_i-\eta_i\left(A\theta_i+b+Z_i\right), 
\end{align*}
where $\{Z_i\}_{i\ge1}$ is i.i.d. standard $\alpha$-stable random variable whose characteristic function is $\EB\exp(it Z)=\exp(-|t|^{\alpha})$. In the following, we discuss the simulation results with constant learning rate $\eta_i=\eta$ and decaying learning rate $\eta_i=i^{-\rho}$.

\paragraph*{Constant learning rate simulation results}
In this part, we aim to verify the result of Theorem~\ref{thm: constant}. In Fig.~\ref{fig: constant_alpha}, we test the sample paths with different $\alpha\in[1.2, 1.4, 1.6, 1.8]$ and regard the difference of the scaled errors between two steps as a jump if it exceeds 1. It is worth noticing that the jumps will merge as $\alpha$ enlarges. It matches our expecation as the jump size of an $\alpha$-stable L\'evy process will decrease as $\alpha$ enlarges.

\paragraph*{Decaying learning rate simulation results}
In this part, we aim to verify the result of Theorem~\ref{thm: asymp_decay}. In Fig.~\ref{fig: decay_alpha}, we plot the histogram of the scaled error in~\eqref{eq: decay_error}. We notice that the empirical distribution of the scaled error matches the asymptotic distribution of $Z_{\infty}$ in~\eqref{eqn: z_stationary}. And we also calculate the 
$95\%$ coverage rate of the scaled error and find it reaches the $95\%$ symmetric confidence interval of $Z_{\infty}$ in Fig.~\ref{fig: decay_coverage}. 

\begin{figure}[t!]
    \centering
    \includegraphics[width=1.0\linewidth]{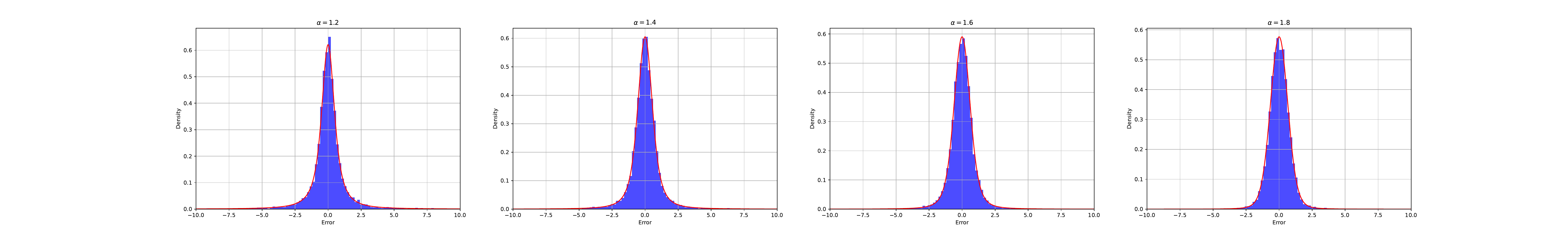}
    \caption{Ordinary Least Squares with decay learning rate $\eta_n=n^{-\rho}$ ($\rho=0.6$) and iterations $n=10^3$ and replications $k=10^4$.  The x-axis represents the scaled error, as defined in \eqref{eq: decay_error}. The y-axis represents the empirical density of the scaled error. Title $\alpha$ indicates the noise is generated from an $\alpha$-stable distribution. Red curve is the density of the stationary distribution in \eqref{eqn: z_stationary}.}
    \label{fig: decay_alpha}
\end{figure}

\begin{figure}[t!]
    \centering
    \includegraphics[width=0.5\linewidth]{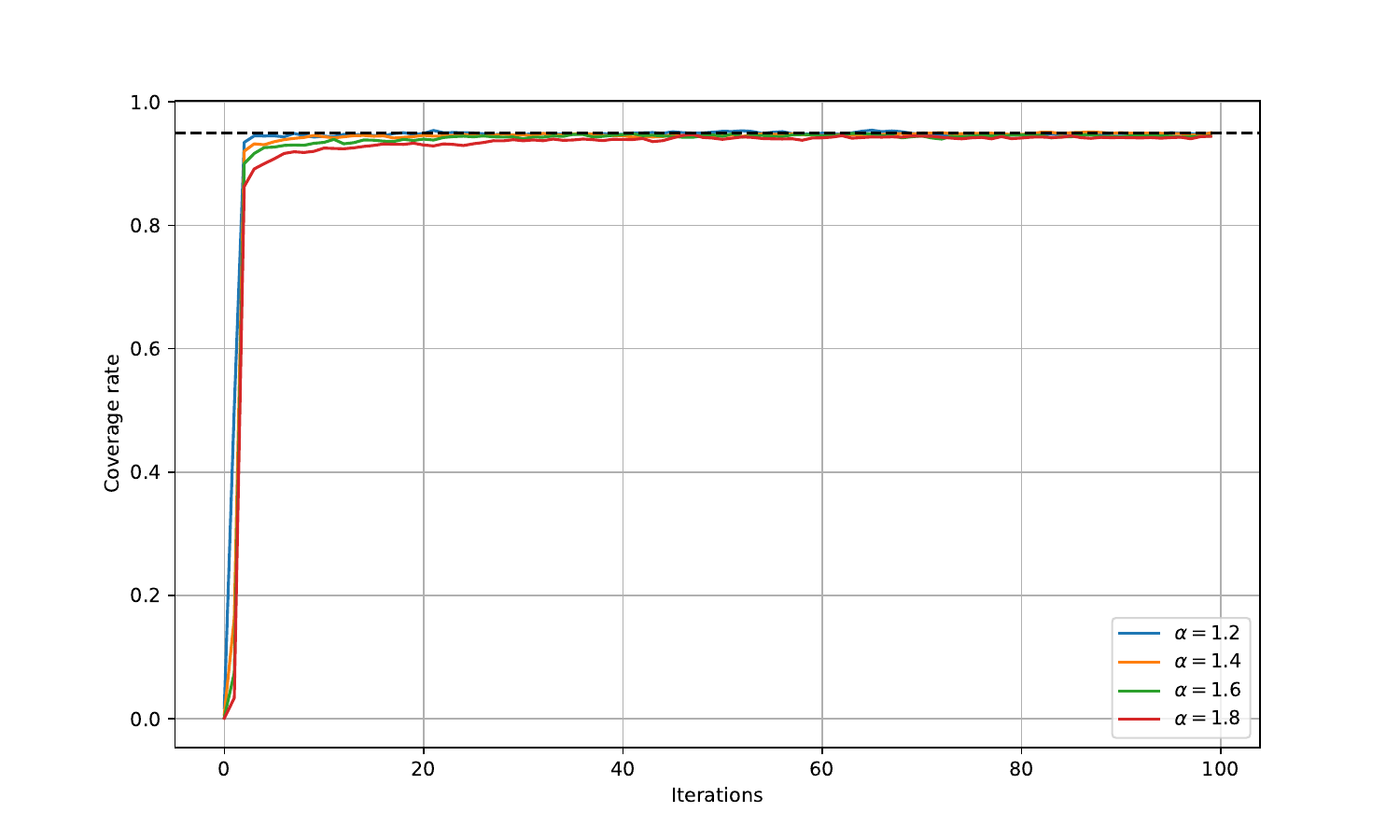}
    \caption{Ordinary Least Squares with decay learning rate $\eta_n=n^{-\rho}$ ($\rho=0.6$) and iterations $n=10^2$ and replications $k=10^4$.  The x-axis represents the number of iterations. The y-axis represents the $95\%$ coverage rate of the scaled error \eqref{eq: decay_error}. The black dashed line represents $y=0.95$.}
    \label{fig: decay_coverage}
\end{figure}

\subsection{Logistic Regression}

\begin{figure}[t!]
    \centering
    \includegraphics[width=1.0\linewidth]{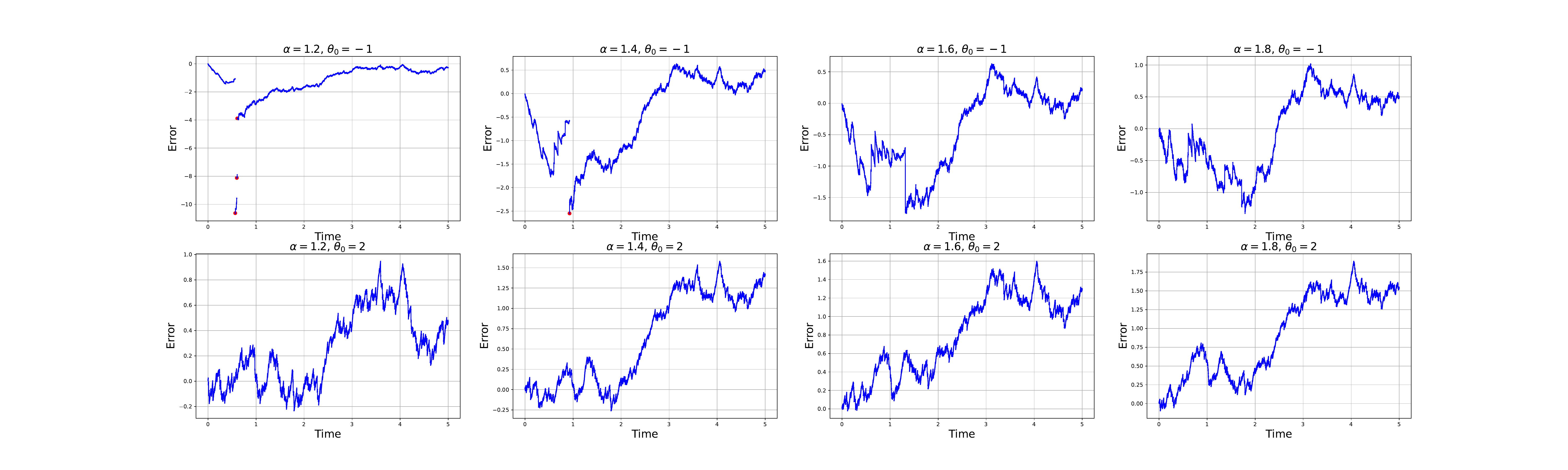}
    \caption{Logistic regression ($\theta^*=1$) with constant learning rate $\eta=10^{-3}$ and iterations $n=5\times 10^3$.  The x-axis represents time, defined as the product of the number of iterations and the learning rate. The y-axis represents the scaled error, as defined in \eqref{eq: const_error}. Title $\alpha$ indicates the covariates $x_i$ are generated from an $\alpha$-stable distribution. The upper figures are the results with initial point $\theta_0=-1$ and the lower figures are the results with initial point $\theta_0=2$. {\rv For the upper figure, we scale the error with rate $\eta^{\frac{1}{\alpha}-1}$ when $\bar{\theta}(t)\le0$ and scale the error with $\eta^{-\frac{1}{2}}$ when $\bar{\theta}(t)>0$. For the lower figure, we scale the error with rate $\eta^{-\frac{1}{2}}$.}}
    \label{fig: constant_alpha_log}
\end{figure}

For the logistic regression problem, we consider solving it in 1 dimension in Section~\ref{subsec: lr}, where we consider $\{x_i\}_{i=1}^n$ are i.i.d. standard $\alpha$-stable random variable whose characteristic function is $\EB\exp(itZ)=\exp(-|t|^{\alpha})$ and model $\theta^*=1$. For $\theta^*=1$, it means the true model has a preference based on the covariates. In the following, we discuss the simulation results with constant learning rate $\eta_i=\eta$ and decaying learning rate $\eta_i=i^{-\rho}$.

\paragraph*{Constant learning rate simulation results}
In this part, we aim to verify the result of Theorem~\ref{thm: constant}. In Fig.~\ref{fig: constant_alpha_log}, we test the sample paths with different $\alpha\in[1.2, 1.4, 1.6, 1.8]$ and regard the difference of the scaled errors between two steps as a jump if it exceeds 1. In the case of $\theta^*=1$, we consider two different initial points $\theta_0=-1$ and $\theta_0=2$. When the initial point $\theta_0=-1$, the solution path will cross the point $0$. By Theorem~\ref{thm: logistic_measure} {\rv and Corollary ~\ref{cor: log_1d}}, we can find the L\'evy measure will evolve from $\nu(\{-1\})=1$ to {\rv $\nu(\{-1,1\})=0$} as the sign of $\theta$ changes. Thus, in the upper line of Fig.~\ref{fig: constant_alpha_log}, we could find there is a {\rv two-stage performance}. {\rv In the first stage, the fluctuations are relatively large due to the jumps induced by heaviness. Later, the fluctuations are relatively small because the stochastic graident is not heavy at this stage. }When the initial point $\theta_0=2$, the solution path will not cross the point $0$ and the sign of $\theta$ does not change. {\rv Thus, according to Corollary~\ref{cor: log_1d}, the trajectory behaves continuously without jumps because of degeneration observation. } The abrupt change is not observed in the lower line of Fig.~\ref{fig: constant_alpha_log}.

\paragraph*{Decaying learning rate simulation results}
In this part, we aim to verify the result of Theorem~\ref{thm: asymp_decay}. In Fig.~\ref{fig: decay_alpha_log}, we plot the histogram of the scaled error in~\eqref{eq: decay_error}. {\rv As $\theta^*=1$, by Corollary~\ref{cor: log_1d}, the stochastic gradient will be light-tailed when $\theta_n\to\theta^*$. Thus, the error should be scaled with $\eta_n^{-\frac{1}{2}}$ and the limit distribution is a Gaussian distribution. In this case, we can use Monte Carlo method to estimate the characteristics of limit distribution and we plot the corresponding density in Figure~\ref{fig: decay_alpha_log} in red. } {\rv 
It is worth noticing that the errors for different $\alpha$ are all scaled with rate $\eta_n^{-\frac{1}{2}}$ because of Corollary~\ref{cor: log_1d}. And the empirical distributions match the expected densities.}

\begin{figure}[t!]
    \centering
    \includegraphics[width=1.0\linewidth]{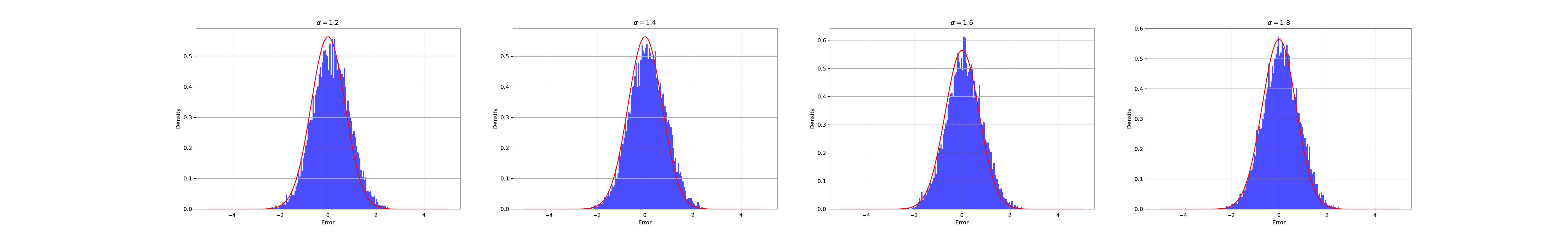}
    \caption{Logistic regression $(\theta^*=1)$ with decay learning rate $\eta_n=n^{-\rho}$ ($\rho=0.6$) and iterations $n=10^3$ and replications $k=10^4$.  The x-axis represents the scaled error, as defined in \eqref{eq: decay_error}. The y-axis represents the empirical density of the scaled error. Title $\alpha$ indicates the covariates $x_i$ are generated from an $\alpha$-stable distribution. {\rv All the errors are scaled with rate $\eta_n^{-\frac{1}{2}}$.}}
    \label{fig: decay_alpha_log}
\end{figure}

\section{Proofs}
\label{sec: proof}

\subsection{\rv Proof Sketch for Theorem~\ref{thm: constant}}
{\rv 
\paragraph*{\textbf{Step 1: Intuition}}
We begin by recalling the deterministic gradient flow:
\begin{align*}
    d\bar{\theta}(t) = -\nabla\ell(\bar{\theta}(t))dt.
\end{align*}
In the stochastic setting, when the stochastic gradient noise is $\alpha$-regularly varying, the dynamics can be heuristically approximated by
\begin{align*}
    d\theta(t) = -\nabla\ell(\theta(t))dt+dL_t,
\end{align*}
where $L_t$ is a L\'evy process indexed by $\alpha$. Let $\Delta(t) = \theta(t)-\bar{\theta}(t)$ denote the error process. Subtracting the two dynamics yields
\begin{align}
\label{eq: error_dynamic}
    d\Delta(t ) &= -\left(\nabla\ell(\theta(t))-\nabla\ell(\bar{\theta}(t))\right)dt+dL_t\notag\\
    &\approx-\nabla^2\ell(\bar{\theta}(t))\Delta(t)dt+dL_t,
\end{align}
where the last line follows from a first-order Taylor expansion. 

\paragraph*{\textbf{Step 2: Approximation}}In the constant stepsize $\eta$ case, we discretize the time interval $[0,T]$ into $\eta, 2\eta, \cdots \lfloor \frac{T}{\eta}\rfloor \eta$ and run both SGD and GD in discrete time. As $\eta\to0$, the discrete trajectories $\theta_{\eta}(\cdot)$ (SGD) and $\bar{\theta}_{\eta}(\cdot)$ (GD) converge to their continuous counterparts $\theta(\cdot)$ and $\bar{\theta}(\cdot)$ above. The proof proceeds by decomposing the error  $\theta_{\eta}(\cdot)-\bar{\theta}_{\eta}(\cdot)$ and controlling two key approximation errors:
\begin{itemize}
    \item[(a)] Drift approximation: the error in approximating $ \nabla\ell(\theta_{\eta}(t))-\nabla\ell(\bar{\theta}_{\eta}(t))$ by $\nabla^2\ell(\bar{\theta}(t))\Delta(t)dt$ via Taylor expansion;
    \item[(b)] Noise approximation: the error in approximating the cumulative noise term by a L\'evy process, which is handled via the functional central limit theorem for $\alpha$-regularly varying random vectors.
\end{itemize}
}

\subsection{\rv Proof Sketch for Theorem~\ref{thm: asymp_decay}}
{\rv
\paragraph*{\textbf{Step 1: Intuition}}
The intuition parallels that of Theorem~\ref{thm: constant}, but here we compare the iterates directly to the true minimizer $\theta^*$. Let $\Delta^\dagger(t) = \theta(t)-\theta^*$. By Taylor expansion around $\theta^*$, when $t$ is large, the dynamics are heuristically
\begin{align*}
    d\Delta^\dagger(t) \approx -\nabla^2\ell(\theta^*)\Delta^\dagger(t) dt +dL_t.
\end{align*}
Thus, in the decaying step-size setting with step sizes $\eta_n$, we expect that, after appropriate scaling, $\theta_n-\theta^*$ and $n\to\infty$, converges to the stationary solution of the SDE:
\begin{align*}
        dZ_t= -\nabla^2\ell(\theta^*)Z_t dt +dL_t.
\end{align*}

\paragraph*{\textbf{Step 2: Time change}}
To compare the discrete-time SGD process to this continuous-time limit, we denote $s_n = \sum_{k=0}^n \eta_k$ and construct a continuous-time interpolated process:
\begin{align*}
    X_t =\eta_n^{\frac{1}{\alpha}}(\theta_n - \theta^*) - (t-s_{n-1})\nabla^2\ell(\theta^*)\eta_n^{\frac{1}{\alpha}}(\theta_n - \theta^*)
\end{align*}
for $t\in[s_{n-1},s_n)$ and a jump exists at each $s_n$. We then show that $X_t$ converges in distribution to the stationary solution of the above SDE as $t\to\infty$.

\paragraph*{\textbf{Step 3: Tightness discussion}}
In step 2, the main technical difficulty to prove the weak convergence lies in establishing the tightness of the sequence $\eta_n^{\frac{1}{\alpha}}(\theta_n - \theta^*)$. We first prove tightness in the special case of symmetric $\alpha$-regularly varying noise, where the characteristic function has only real parts, making the argument straightforward. For the general asymmetric case, we decompose the noise into positive and negative parts, symmetrize each part separately, and apply the symmetric-noise result to each component. Combining these results yields weak convergence in the general case.
}
\subsection{Auxiliary Lemmas}

\begin{lemma}
    \label{lem: dominate_sl}
    For any $u\in\mathbb{S}^{d}$ and $x>0$, we denote:
    \begin{align*}
        b_{0,u}(x,\theta^*)=x^{\alpha}\PB\left(u^\top\nabla\ell(\theta^*,\xii)>x\right).
    \end{align*}
    Then the slowly varying function $b_{0,u}(x,\theta^*)$ satisfies:
    \begin{align*}
        \lim_{x\to+\infty}\frac{b_{0,u}(x,\theta^*)}{b_{0}(x,\theta^*)}=C,
    \end{align*}
    where $C\in[0,+\infty)$.
\end{lemma}
\begin{proof}
    This result is a direct corollary of Theorem 1.1 in \cite{basrak2002characterization}, which shows the slowly varying function of the tail probability of the norm $\|\nabla\ell(\theta^*,\xii)\|$ can dominate any linear combination of the vector $\nabla\ell(\theta^*,\xii)$, which plays a significant role in proving the tightness in {\rv Section~\ref{sec: tight}}.     
\end{proof}

\begin{lemma}
\label{lem: lip_b}
    Under Assumption~\ref{asmp: heavy_tail}, there exists $M, C>0$ such that for any $x>M$ and $\theta_1,\theta_2\in\RB^d$, we have:
    \begin{align*}
        |b_0(x,\theta_1)-b_0(x,\theta_2)|\le C b_0(x,\theta^*)\|\theta_1-\theta_2\|.
    \end{align*}
    \begin{proof}
        By Assumption~\ref{asmp: heavy_tail}, there exists $M>0$, for any $x>M$, such that:
        \begin{align*}
            \left|\frac{b_0(x,\theta_i)}{b_0(x,\theta^*)}-C(\theta_i)\right|\le\|\theta_1-\theta_2\|, \hspace{4pt}i\in\{1,2\}.
        \end{align*}
        Then we have:
        \begin{align*}
            \frac{|b_0(x,\theta_1)-b_0(x,\theta_2)|}{b_0(x,\theta^*)}&\le \left|\frac{b_0(x,\theta_1)}{b_0(x,\theta^*)}-C(\theta_1)\right|+\left|\frac{b_0(x,\theta_2)}{b_0(x,\theta^*)}-C(\theta_2)\right|+|C(\theta_1)-C(\theta_2)|\\
            &\le C \|\theta_1-\theta_2\|.\notag
        \end{align*}
    \end{proof}
\end{lemma}

\begin{lemma}[Lemma 7 in \cite{wang2021convergence}]
    \label{lem: moment}
    Suppose $p\in[1,2]$ and $\{S_t\}$ is a $d$-dimensional martingale with finite $p$ moment for every $t$, and $S_0=0$. Then:
    \begin{align*}
        \EB\|S_t\|^{p}\le2^{2-p}d^{1-\frac{p}{2}}\sum_{i=1}^{t}\EB\|S_i-S_{i-1}\|^{p}.
    \end{align*}
\end{lemma}

\begin{lemma}
\label{lem: eta_expansion}
    For the sequence $\{\eta_n^{\frac{1}{\alpha}-1}b_1(\eta_n^{-1})\}$, where $\eta_n=c\cdot n^{-\rho}$ and $\rho<1$, it satisfies:
    \begin{align*}
        \frac{\eta_{n+1}^{\frac{1}{\alpha}-1}b_1(\eta_{n+1}^{-1})}{\eta_n^{\frac{1}{\alpha}-1}b_1(\eta_n^{-1})}=1+\OM(n^{-1}).
    \end{align*}
    Moreover, if $\rho=1$, it satisifes:
    \begin{align*}
        \frac{\eta_{n+1}^{\frac{1}{\alpha}-1}b_1(\eta_{n+1}^{-1})}{\eta_n^{\frac{1}{\alpha}-1}b_1(\eta_n^{-1})}=1+\frac{1-\frac{1}{\alpha}}{c}\eta_na_n,
    \end{align*}
    where $\lim_{n\to\infty}a_n=1$.
\end{lemma}
\begin{proof}
    We denote $h(x)=x^{1-\frac{1}{\alpha}}b_1(x)$, which is $\left(1-\frac{1}{\alpha}\right)$-regular varying. Then we have:
    \begin{align}
        \label{eq: eta_0}
        \frac{h(\eta_{n+1}^{-1})}{h(\eta_n^{-1})}=1+\frac{\eta_n^{-1}h'(\xi_{n}^{-1})}{h(\eta_n^{-1})}\left(\frac{\eta_n}{\eta_{n+1}}-1\right),
    \end{align}
    where $\xi_n\in(\eta_{n+1},\eta_n)$. Denote $f(x)=\log h(e^x)$, we have:
    \begin{align}
    \label{eq: eta_1}
        f'(x) = \frac{e^x h'(e^x)}{h(e^x)}=\left(1-\frac{1}{\alpha}\right)+\frac{e^x b_1'(e^x)}{b_1(e^x)}.
    \end{align}
    As $b_1(x)$ is slowly varying, we have $\lim_{x\to\infty}f'(x)=1-\frac{1}{\alpha}$. Moreover, we have:
    \begin{align}
    \label{eq: eta_2}
        \frac{\xi_n^{-1}}{\eta_n^{-1}}\in\left(1,\frac{\eta_{n+1}^{-1}}{\eta_n^{-1}}\right)=\left(1,(1+n^{-1})^\rho\right),
    \end{align}
    which leads to:
    \begin{align}
        \label{eq: eta_3}
        \frac{h(\xi_{n}^{-1})}{h(\eta_n^{-1})}=1+o(1).
    \end{align}
    Thus, when $\rho<1$, plugging \eqref{eq: eta_1}, \eqref{eq: eta_2} and \eqref{eq: eta_3} into \eqref{eq: eta_0}, we have:
    \begin{align*}
        \frac{h(\eta_{n+1}^{-1})}{h(\eta_n^{-1})}=1+\OM(n^{-1}).
    \end{align*}
    When $\rho=1$, we have:
    \begin{align*}
        \frac{h(\eta_{n+1}^{-1})}{h(\eta_n^{-1})}=1+\frac{1-\frac{1}{\alpha}}{c}\eta_n a_n,
    \end{align*}
    where $\lim_{n\to\infty}a_n=1$.
\end{proof}

\begin{lemma}[Lemma 4.2 in \cite{fabian1967stochastic}]
    \label{lem: recursion}
    Supposing $\alpha\in(0,1]$ and $\beta>0$ and $C$ is a  constant and $\lim_{n\to\infty}C_n=c>0$. For a positive sequence satisfying:
    \begin{align*}
        A_{n+1}=(1-C_n n^{-\alpha})A_n+C n^{-\alpha-\beta},
    \end{align*}
    we have $\sup_{n}n^{\beta}A_n<\infty$ when $\alpha<1$. If moreover, we have $c-\beta>0$, then we also have $\sup_{n}n^{\beta}A_n<\infty$ when $\alpha=1$.
\end{lemma}

\begin{lemma}
    \label{lem: unif_integral}
    Set $x=n$ in Assumption~\ref{asmp: heavy_tail} and denote $\mu_n(\theta,\cdot)=n\PB\left(\frac{\nabla \ell(\theta,\xii)}{a_n}\in \cdot\right)$, where $a_n=n^{\frac{1}{\alpha}}b_1(n)$, we have $\mu_n(\theta,\cdot)$ is uniformly {\rv integrable} w.r.t. $n$:
    \begin{align*}
        \lim_{t\to\infty}\sup_n\EB_{\mu_n}\|X\|\1(\|X\|>t)=0.
    \end{align*}
\end{lemma}
\begin{proof}
    As $\theta$ is fixed and $b_0(x,\theta)\sim b_0(x,\theta^*)$ when $x\to\infty$, we ignore the dependence in the proof. In fact, we have:
    \begin{align*}
        \EB_{\mu_n}\|X\|\1(\|X\|>t)&=\frac{n}{a_n}\int_{ta_n}^{\infty}|x|\PB\left(\|\nabla\ell(\theta,\xii)\|\in dx\right)\\
        &=\frac{n}{a_n}\int_{ta_n}^{\infty}|x|^{-\alpha}(b_0'(x)x-\alpha b_0(x))dx\notag\\
        &\sim\frac{n(b_0'(ta_n)ta_n-\alpha b_0(ta_n))}{t^{\alpha-1}a_n^{\alpha}},\notag
    \end{align*}
    where the last step is due to Karamata’s theorem. We notice that $\frac{n}{a_n^{\alpha}}\sim b_0^{-1}(a_n)$ {\rv by the definition $\PB(\|\nabla\ell(\theta,\xii)\|>t)=\frac{b_0(t)}{t^\alpha}$}. Then, we have:
    \begin{align*}
        \EB_{\mu_n}\|X\|\1(\|X\|>t)\sim\frac{(b_0'(ta_n)ta_n-\alpha b_0(ta_n))}{t^{\alpha-1}b_0(a_n)}.
    \end{align*}
    By the fact that $\frac{b'_0(x)x}{b_0(x)}=o( 1)$ when $x$ is large, we have:
    \begin{align*}
        \sup_n \EB_{\mu_n}\|X\|\1(\|X\|>t)=\OM(t^{1-\alpha}).
    \end{align*}
    Then, our result is concluded.
\end{proof}
\begin{lemma}
    \label{lem: lip_expectation}
    Set $x=n$ in Assumption~\ref{asmp: heavy_tail} and denote $\mu_n(\theta,\cdot)=n\PB\left(\frac{\nabla \ell(\theta,\xii)}{a_n}\in \cdot\right)$, where $a_n=n^{\frac{1}{\alpha}}b_1(n)$. For any $\theta_1$ and $\theta_2$ and $t>0$ , we have:
    \begin{align*}
        \left\|\EB_{\mu_n(\theta_1)}X\1(\|X\|>t)-\EB_{\mu_n(\theta_2)}X\1(\|X\|>t)\right\|=\mathcal{O}(\|\theta_1-\theta_2\|).
    \end{align*}
\end{lemma}
\begin{proof}
    In fact, we have:
    \begin{align*}
        &\left\|\EB_{\mu_n(\theta_1)}X\1(\|X\|>t)-\EB_{\mu_n(\theta_2)}X\1(\|X\|>t)\right\|\\
        \le&\frac{n}{a_n}\int_{ta_n}^{\infty}|x||\PB\left(\|\nabla\ell(\theta_1,\xii)\|\in dx\right)-\PB\left(\|\nabla\ell(\theta_2,\xii)\|\in dx\right)|\notag\\
        \le&\frac{n}{a_n}\int_{ta_n}^{\infty}|x|^{-\alpha}(|b_0'(x,\theta_1)-b_0'(x,\theta_2)|x+\alpha |b_0(x,\theta_1)-b_0(x,\theta_2)|)dx\notag\\
        \sim&\|\theta_1-\theta_2\|,\notag
    \end{align*}
    where the last step is due to Lemma~\ref{lem: lip_b} and Karamata's theorem.
\end{proof}

\begin{lemma}
\label{lem: ito}
    Suppose $L_{\cdot}$ is a semimartingale, the s.d.e. $dZ_t = - F(t)Z_t dt+dL_t$ admits a solution:
    \begin{align*}
        Z_t = \Phi(t)(Z_0+\int_0^t \Phi(s)^{-1}dL_s),
    \end{align*}
    where $\Phi(t)$ is the solution of the following o.d.e.:
    \begin{align*}
        \begin{cases}
            &d\Phi(t)=-\nabla F(t)\Phi(t)dt,\\
            &\Phi(0)=I_d.
        \end{cases}
    \end{align*}
\end{lemma}
\begin{proof}
    This is a direct result by Th\'eor\`eme 6.8 p.194~\cite{jacod2006calcul}. Or more directly, denote $h(t,u)=\Phi(t)^{-1}u$, by Ito's formula, we have:
    \begin{align*}
        h(t,Z_t)-h(0,Z_0)=&\int_{0}^t d\Phi(s)^{-1}\cdot Z_s+\int_0^t \Phi(s)^{-1}dZ_s\\
        &+\sum_{0<s\le t}\left(h(s,Z_s)-h(s,Z_{s^-})-\frac{\partial h(s,Z_{s^-})}{\partial u}(Z_s-Z_{s^-})\right),\notag
    \end{align*}
    where the last term is 0 as $h(t,u)$ is affine in $u$. By $d\Phi(s)^{-1}=\Phi(s)^{-1}(d\Phi(s)) \Phi(s)^{-1}$, we have:
    \begin{align*}
        h(t,Z_t)-h(0,Z_0)=\int_0^t\Phi(s)^{-1}dL_t.
    \end{align*}
\end{proof}

\begin{lemma}[Theorem 11.6.6 in \cite{whitt2002stochastic}]
\label{lem: cadlag_converge}
    A sequence of $\cadlag$ process $\{X_{t}^n\}_{n=1}^{\infty}$ converge in law to a $\cadlag$ process $X$ w.r.t. $J_1$ topology iff $\{X_{t}^n\}_{n=1}^{\infty}$ is tight and all finite dimensional distributions at times where $\{X_{t}^n\}_{n=1}^{\infty}$ are continuous
    converge to those of X.
\end{lemma}

\begin{lemma}[Tightness, Theorem 13.2 in \cite{billingsley2013convergence}]
\label{lem: cadlag_tight}
    A sequence of $\cadlag$ process $\{X_{t}^n\}_{n=1}^{\infty}$ is tight w.r.t. $J_1$ iff:
    \begin{itemize}
        \item $\{X_{t}^n\}_{n=1}^{\infty}$ is uniformly bounded in probability:
        \begin{align*}
            \lim_{C\to \infty}\limsup_{n\to\infty}\PB\left(\sup_{0\le t\le T}\|X_t^n\|\ge C\right)=0.
        \end{align*}
        \item For any $\varepsilon>0$,
        \begin{align*}
            \lim_{\delta\downarrow0}\limsup_{n\to\infty}\PB\left(\omega_{\delta}'(X^n_T)\ge\varepsilon\right)=0,
        \end{align*}
        where:
        \begin{align*}
            \omega_{\delta}'(f) = \inf_{\Pi_n}\sup_{i=1}^n\sup_{s,t\in[t_{i-1},t_i)}\|f(t)-f(s)\|,
        \end{align*}
        and $\Pi_n$ is a finite partition on $[0,T]$.
    \end{itemize}
\end{lemma}

\begin{lemma}[Theorem 1 in \cite{gnedenko1968limit}, Chapter 25]
    \label{lem: infinitesimal}
    In order that for some suitably chosen constants $A_n$, the distributions of the sums
    \begin{align*}
        \zeta_n=\sum_{i=1}^{k_n}\xi_{n,i}-A_n
    \end{align*}
    of independent infinitesimal random variables converge to a limit, it {\rv is} necessary and sufficient that there exists nondecreasing functions $M(u)$ and $N(u)$, where $M(-\infty)=0$ and $N(+\infty)=0$, defined in the intervals $(-\infty,0)$ and $(0,+\infty)$ respectively, and a constant $\sigma\ge0$, such that:
    \begin{itemize}
        \item (1) At every continuity point of $M(u)$ and $N(u)$,
        \begin{align*}
            &\lim_{n\to\infty}\sum_{i=1}^{k_n}F_{n,i}(u)=M(u),\hspace{4pt}(u<0),\\
            &\lim_{n\to\infty}\sum_{i=1}^{k_n}(F_{n,i}(u)-1)=N(u),\hspace{4pt}(u>0),
        \end{align*}
        {\rv where $F_{n,i}(u)=\PB\left(\xi_{n,i}\le u\right)$.}
        \item (2) 
        \begin{align*}
            &\lim_{\varepsilon\to0}\varlimsup_{n\to\infty}\sum_{i=1}^{k_n}\left\{\int_{|x|<\varepsilon}x^2dF_{n,i}(x)-\left(\int_{|x|<\varepsilon}x dF_{n,i}(x)\right)^2\right\}\\
            =&\lim_{\varepsilon\to0}\varliminf_{n\to\infty}\sum_{i=1}^{k_n}\left\{\int_{|x|<\varepsilon}x^2dF_{n,i}(x)-\left(\int_{|x|<\varepsilon}x dF_{n,i}(x)\right)^2\right\}=\sigma^2\notag
        \end{align*}
    \end{itemize}
\end{lemma}

\subsection{Proof of Theorem~\ref{thm: constant}}
For a positive real number $s$, we denote $[s]_{\eta}:=\lfloor s/\eta\rfloor\eta$. For a fixed horizon $T$, we have:
\begin{align*}
    \theta_{\eta}([T]_{\eta})-\theta_{\eta}(0)&=-\eta\sum_{k=0}^{\lfloor T/\eta\rfloor-1}\nabla \ell(\theta_{\eta}(k\eta))+\eta \sum_{k=0}^{\lfloor T/\eta\rfloor-1}\nabla \ell(\theta_{\eta}(k\eta))-\nabla \ell(\theta_{\eta}(k\eta),\xii_{(k+1)\eta})\\
    &:=-\eta\sum_{k=0}^{\lfloor T/\eta\rfloor-1}\nabla \ell(\theta_{\eta}(k\eta))+\eta \sum_{k=0}^{\lfloor T/\eta\rfloor-1}M(\theta_{\eta}(k\eta),\xii_{(k+1)\eta})\notag\\
    \bar{\theta}(T)-\bar{\theta}(0)&=-\sum_{k=0}^{\lfloor T/\eta\rfloor-1}\int_{k\eta}^{(k+1)\eta}\nabla \ell(\bar{\theta}(t))dt.
\end{align*}
It is worth noticing that $\sum_{k=0}^{n}M(\theta_{\eta}(k\eta),\xii_{(k+1)\eta})$ is a martingale. Then, we have the difference:
\begin{align}
    \label{eq: error_decomp}
    \theta_{\eta}([T]_{\eta})-\bar{\theta}(T)=&-\sum_{k=0}^{\lfloor T/\eta\rfloor-1}\left(\nabla\ell(\theta_{\eta}(k\eta))\eta-\int_{k\eta}^{(k+1)\eta}\nabla\ell(\bar{\theta}(t))dt\right)\\
    &+\eta \sum_{k=0}^{\lfloor T/\eta\rfloor-1}\left(M(\theta_{\eta}(k\eta),\xii_{(k+1)\eta})-M(\bar{\theta}(k\eta),\xii_{(k+1)\eta})\right)\notag\\
    &+\eta \sum_{k=0}^{\lfloor T/\eta\rfloor-1}M(\bar{\theta}(k\eta),\xii_{(k+1)\eta}).\notag
\end{align}
Multiplying $\eta^{\frac{1}{\alpha}-1}b_1(\eta^{-1})$ both sides, we have:
\begin{align}
\label{eq: scaled_error}
    &\eta^{\frac{1}{\alpha}-1}b_1(\eta^{-1})\left(\theta_{\eta}([T]_{\eta})-\bar{\theta}(T)\right)\\
    =&-\eta^{\frac{1}{\alpha}-1}b_1(\eta^{-1})\sum_{k=0}^{\lfloor T/\eta\rfloor-1}\left(\nabla\ell(\theta_{\eta}(k\eta))\eta-\int_{k\eta}^{(k+1)\eta}\nabla\ell(\bar{\theta}(t))dt\right)\notag\\
    &+\eta^{\frac{1}{\alpha}}b_1(\eta^{-1}) \sum_{k=0}^{\lfloor T/\eta\rfloor-1}\left(M(\theta_{\eta}(k\eta),\xii_{(k+1)\eta})-M(\bar{\theta}(k\eta),\xii_{(k+1)\eta})\right)\notag\\
    &+\eta^{\frac{1}{\alpha}}b_1(\eta^{-1}) \sum_{k=0}^{\lfloor T/\eta\rfloor-1}M(\bar{\theta}(k\eta),\xii_{(k+1)\eta})\notag
\end{align}
In the following, we study limiting behaviors of above three terms separately. Prior to that, we give some basic facts about the error $\theta_{\theta}(\cdot)-\bar{\theta}(\cdot)$.
\begin{lemma}
    \label{lem: finite_rate}
    For any $p\in[1,\alpha)$ and $t\in[0,T]$, we have:
    \begin{align*}
        \mathbb{E}\|\theta_{\eta}([t]_{\eta})-\bar{\theta}(t)\|^p=\OM\left(\eta^{p-1}\right).
    \end{align*}
\end{lemma}
\begin{proof}
    We denote $H_{\eta}(t):=\eta^{\frac{1}{p}-1}\left(\theta_{\eta}([t]_{\eta})-\bar{\theta}(t)\right)$. Then, by error decomposition~\eqref{eq: error_decomp}, we have:
    \begin{align*}
        \|H_\eta(t)\|\le C_{\text{Lip}}\int_0^t\|H_{\eta}(s)\|ds+\eta^{\frac{1}{p}}\left\| \sum_{k=0}^{\lfloor t/\eta\rfloor-1}M(\theta_{\eta}(k\eta),\xii_{(k+1)\eta})\right\|,
    \end{align*}
    where it holds by $\nabla \ell(\theta,\xii)$ is Lipschitz w.r.t. $\theta${\rv, and $C_{\text{Lip}}$ is from Assumption~\ref{asmp: g_lip}}. Taking $p$-th moment, we have:
    \begin{align*}
        \mathbb{E}\|H_\eta(t)\|^p&\le2^{p-1}C_{\text{Lip}}^p\EB\left(\int_0^t\|H_{\eta}(s)\|ds\right)^p+2^{p-1}\eta \mathbb{E}\left\| \sum_{k=0}^{\lfloor t/\eta\rfloor-1}M(\theta_{\eta}(k\eta),\xii_{(k+1)\eta})\right\|^p\\
        &\le(2t)^{p-1}C_{\text{Lip}}^p\int_0^t\EB\|H_{\eta}(s)\|^pds+2t d^{1-\frac{p}{2}}C_p,\notag
    \end{align*}
    where the last inequality holds by Lemma~\ref{lem: moment} and the fact $p$-th moment of $\nabla \ell(\theta,\xii)$ is finite {\rv and $C_p=\EB\|M(\theta,\xii)\|^p$}. Multiplying $\exp\left(-\frac{(2C_{\text{Lip}})^{p}}{2p}t^{p}\right)$ both sides and integrating over $t$, we have:
    \begin{align*}
        \exp\left(-\frac{(2C_{\text{Lip}})^{p}}{2p}T^{p}\right)\int_{0}^T\EB\|H_{\eta}(t)\|^{p}dt\le\int_{0}^T\exp\left(-\frac{(2C_{\text{Lip}})^{p}}{2p}t^{p}\right)tdt\cdot 2d^{1-\frac{p}{2}}C_p.
    \end{align*}
    Then we conclude $\EB\|H_{\eta}(t)\|^{p}=\OM(1)$.
\end{proof}
Now, we turn to study each terms in Equation~\eqref{eq: scaled_error}. In a high level idea, the first component can be approximated by the Hessian matrix $\nabla^2 \ell$ with additional small terms (Lemma~\ref{lem: hessian_approximate}). The second component would be zero as $M(\theta,\xii)$ is almost surely Lipschitz w.r.t. $\theta$ (Lemma~\ref{lem: error_diff}). The third component is the summation of independent random variables, which can be characterized by central limit theorem with the heavy-tail version (Lemma~\ref{lem: additive_converge}).
\begin{lemma}
\label{lem: hessian_approximate}
    Denote $Z_{\eta}(t):=\eta^{\frac{1}{\alpha}-1}b_1(\eta^{-1})\left(\theta_{\eta}([t]_{\eta})-\bar{\theta}(t)\right)$ for any $t\in[0,T]$, as $\eta\to0$, we have:
    \begin{align*}
        \eta^{\frac{1}{\alpha}-1}b_1(\eta^{-1})\sum_{k=0}^{\lfloor t/\eta\rfloor-1}\left(\nabla\ell(\theta_{\eta}(k\eta))\eta-\int_{k\eta}^{(k+1)\eta}\nabla\ell(\bar{\theta}(s))ds\right)=\int_0^t\nabla^2\ell(\bar{\theta}(s))Z_\eta(s)ds+\delta_{\eta}(t),
    \end{align*}
    where $\EB\sup_{t\le T}\|\delta_{\eta}(t)\|=o(1)$.
\end{lemma}
\begin{proof}
    When $s\in[k\eta,(k+1)\eta]$, we denote:
    \begin{align*}
        Y_{\eta}(s)=:\eta^{\frac{1}{\alpha}-1}b_1(\eta^{-1})\left(\nabla \ell(\theta_{\eta}(k\eta))-\nabla\ell(\bar{\theta}(s))\right)-\nabla^2\ell(\bar{\theta}(s))Z_{\eta}(s).
    \end{align*}
    By {\rv Assumption}~\ref{asmp: q_expansion}, we have $\|Y_{\eta}(s)\|=O(\eta^{\frac{1}{\alpha}-1}b_1(\eta^{-1})\|\theta_{\eta}(k\eta)-\bar{\theta}(s)\|^q)$. Then, we have:
    \begin{align*}
        \EB\int_{0}^T\|Y_{\eta}(s)\|ds=O\left(\eta^{\frac{1}{\alpha}-1}b_1(\eta^{-1})\cdot\eta^{q-1}\right)=o(1).
    \end{align*}
    Thus, $\EB\sup_{t\le T}\|\delta_{\eta}(t)\|\le\EB\int_0^T \|Y_{\eta}(s)\|ds=o_{\PB}(1)$.
\end{proof}

\begin{lemma}
\label{lem: error_diff}
    As $\eta\to0$, for any $p\in(\frac{2\alpha}{\alpha+1},\alpha)$, we have:
    \begin{align*}
        \EB\sup_{t\le T}\left\|\eta^{\frac{1}{\alpha}}b_1(\eta^{-1}) \sum_{k=0}^{\lfloor t/\eta\rfloor-1}\left(M(\theta_{\eta}(k\eta),\xii_{(k+1)\eta})-M(\bar{\theta}(k\eta),\xii_{(k+1)\eta})\right)\right\|^p=o(1).
    \end{align*}
\end{lemma}
\begin{proof}
We denote $S_n=\eta^{\frac{1}{\alpha}}b_1(\eta^{-1})\sum_{k=0}^{n-1}M(\theta_{\eta}(k\eta),\xii_{(k+1)\eta})-M(\bar{\theta}(k\eta),\xii_{(k+1)\eta})$. It can be verified $\{S_n\}_{n\ge1}$ is a martingale. Then, for every $1<p<\alpha$, we have:
\begin{align*}
    \mathbb{E}\|S_n\|^p&\le2^{2-p}d^{1-\frac{p}{2}} \eta^{\frac{p}{\alpha}}b_1^p(\eta^{-1}) \sum_{k=0}^{n-1}\mathbb{E}\|M(\theta_{\eta}(k\eta),\xii_{(k+1)\eta})-M(\bar{\theta}(k\eta),\xii_{(k+1)\eta})\|^p\\
    &\le 2d^{1-\frac{p}{2}} \eta^{\frac{p}{\alpha}}b_1^p(\eta^{-1}) \sum_{k=0}^{n-1}\mathbb{E}\|\nabla \ell(\theta_{\eta}(k\eta),\xii_{(k+1)\eta})-\nabla\ell(\bar{\theta}(k\eta),\xii_{(k+1)\eta})\|^p\notag\\
    &\le2C_{\text{Lip}}d^{1-\frac{p}{2}} \eta^{\frac{p}{\alpha}}b_1^p(\eta^{-1}) \sum_{k=0}^{n-1}\mathbb{E}\|\theta_{\eta}(k\eta)-\bar{\theta}(k\eta)\|^p\notag.
\end{align*}
By Lemma~\ref{lem: finite_rate} and taking $n=\lfloor t/\eta\rfloor$ and $p\in(\frac{2\alpha}{\alpha+1},\alpha)$, we have:
\begin{align*}
    \mathbb{E}\|S_{\lfloor t/\eta\rfloor}\|^p=\OM\left(\eta^{\frac{p}{\alpha}+p-2}b_1^p(\eta^{-1})\right).
\end{align*}
As $b_1(x)$ is slowly varying function, we have $\mathbb{E}\|S_{\lfloor t/\eta\rfloor}\|^p=o(1)$. Then, by Doob's martingale inequality, we have:
\begin{align*}
    \EB\sup_{0\le t\le T}\|S_{\lfloor t/\eta\rfloor}\|^p\le\left(\frac{p}{p-1}\right)^p\EB\|S_{\lfloor T/\eta\rfloor}\|^p=o(1).
\end{align*}
\end{proof}

\begin{lemma}
\label{lem: additive_converge}
    As $\eta\to0$, in $D[0,T]$ endowed with the $J_1$ topology, we have:
    \begin{align*}
        L_{\eta}(\cdot):=\eta^{\frac{1}{\alpha}}b_1(\eta^{-1}) \sum_{k=0}^{\lfloor \cdot/\eta\rfloor-1}M(\bar{\theta}(k\eta),\xii_{(k+1)\eta})\Rightarrow L(\cdot),
    \end{align*}
    where $L(\cdot)$ is an additive process with characteristic triple $(0,\nu_t,-\gamma_t)$, where:
    \begin{align*}
        \nu_t(dx)&=\int_{s\in[0,t]}\nu(\bar{\theta}(s),dx)ds,\\
        \gamma_t &= \int_{s\in[0,t],\|x\|>1}x\nu(\bar{\theta}(s),dx)ds.
    \end{align*}
\end{lemma}
\begin{proof}
    {\rv This result is an extension from a similar result with i.i.d. case in Theorem 7.1 \cite{resnick2007heavy}}. We denote $H_k(\eta):=\eta^{\frac{1}{\alpha}}b_1(\eta^{-1}) \nabla\ell(\bar{\theta}(k\eta),\xii_{(k+1)\eta})$ as the scaled error and consider the following truncated error:
    \begin{align}
    \label{eq: v_eta_varepsilon}
        V^{\varepsilon}_{\eta}=\sum_{k=0}^{\lfloor T/\Delta\rfloor-1}H_k(\eta)\1\left(\left\|H_k(\eta)\right\|>\varepsilon\right)-\EB\left[ H_k(\eta)\1\left(\varepsilon<\left\|H_k(\eta)\right\|\le1\right)\right].
    \end{align}
    We also denote the original truncated error as:
    \begin{align*}
        V_{\eta}=\sum_{k=0}^{\lfloor T/\Delta\rfloor-1}H_k(\eta)-\EB\left[ H_k(\eta)\1\left(\left\|H_k(\eta)\right\|\le1\right)\right].
    \end{align*}
    In the following, our proof strategy is 3-staged: $V_{\eta}^{\varepsilon}\Rightarrow V^{\varepsilon}$ as $\eta\to0$ , $V^{\varepsilon}\Rightarrow V$ as $\varepsilon\to0$, and $V_{\eta}$ and $V_{\eta}^{\varepsilon}$ are almost the same. We always assume $\varepsilon$ is not the jump of the function $\nu(\theta,\{\|x\|>t\})$. We notice the first term of $V_{\eta}^{\varepsilon}$ in~\eqref{eq: v_eta_varepsilon} is a transformation $\psi^{\varepsilon}:\boldsymbol{M}_p([0,T]\times \bar{\RB}^d_{\varepsilon'})\to D[0,T]$ ($\varepsilon'<\varepsilon$) of a point process $N_{\eta}(\cdot):$
    \begin{align*}
        N_{\eta}(\cdot,\cdot)&=\sum_{k=0}^{\lfloor T/\eta\rfloor-1}\delta_{k\eta,H_k(\eta)}(\cdot,\cdot),\\
        \psi^{\varepsilon}(N_{\eta})(t)&=\sum_{k\le t}H_k(\eta)\1\left(\|H_k(\eta)\|>\varepsilon\right).
    \end{align*}
    Thus, we first study the convergence for the point process $ N_{\eta}(\cdot)$. For any bounded function $f$ such that $0\le f(t,x)\le C_f\1(\|x\|>\varepsilon)$, the Laplace functional of $N_{\eta}$ is:
    \begin{align*}
        -\ln \EB\exp(-N_{\eta}(f))=-\sum_{k=0}^{\lfloor T/\eta\rfloor-1}\ln \EB\exp\left(-f(k\eta,H_k(\eta))\right).
    \end{align*}
    By $x-1-\frac{(x-1)^2}{x}\le\ln x\le x-1$ for any $x\in(0,1]$, we have:
    \begin{align*}
    0&\le-\ln \EB\exp(-N_{\eta}(f))-        \sum_{k=0}^{\lfloor T/\Delta\rfloor-1}\left(1-\EB\exp\left(-f(k\eta,H_k(\eta))\right)\right)\\
    &\rv\le\sum_{k=0}^{\lfloor T/\Delta\rfloor-1}\frac{\left(1-\EB\exp\left(-f(k\eta, H_k(\eta))\right)\right)^2}{\EB\exp(-f(k\eta,H_k(\eta)))}\notag.
    \end{align*}
    We also notice:
    \begin{align*}
    \rv 1\ge\EB\exp\left(-f(k\eta,H_k(\eta))\right)&\rv\ge\EB\exp\left(-C_f \1(\|H_k(\eta)\|>\varepsilon)\right)\\
        &\rv\ge1-C_f\PB\left(\|H_k(\eta)\|>\varepsilon\right)\notag.
    \end{align*}
    Then we have:
    \begin{align*}
    0&\le-\ln \EB\exp(-N_{\eta}(f))-        \sum_{k=0}^{\lfloor T/\eta\rfloor-1}\left(1-\EB\exp\left(-f(k\eta,H_k(\eta))\right)\right)\\
    &\le\sum_{k=0}^{\lfloor T/\eta\rfloor-1}\frac{C_f^2\PB(\|H_k(\eta)\|>\varepsilon)^2}{1-C_f\PB(\|H_k(\eta)\|>\varepsilon)}\notag.
    \end{align*}
    {\rv
    We also notice that:
    \begin{align*}
        \PB(\|H_k(\eta)\|>\varepsilon)&\overset{(a)}{=}\eta\cdot\frac{b_0\left(\frac{\varepsilon}{\eta^{\frac{1}{\alpha}}b_1(\eta^{-1})},\bar{\theta}(k\eta)\right)b_1\left(\eta^{-1}\right)^\alpha}{\varepsilon^{\alpha}}\\
        &\overset{(b)}{=}\eta\cdot\frac{b_0\left(\frac{\varepsilon}{\eta^{\frac{1}{\alpha}}b_1(\eta^{-1})},\bar{\theta}(k\eta)\right)}{\varepsilon^\alpha b_0\left(\frac{1}{\eta^{\frac{1}{\alpha}}b_1(\eta^{-1})},\theta^*\right)}\\
        &\overset{(c)}{=}\eta\cdot\frac{C(\bar{\theta}(k\eta))}{\varepsilon^\alpha}+\OM(\eta),
    \end{align*}
    where (a) is the definition of Eqn~\eqref{eq: tail_norm}, (b) is the choice of $b_1(\cdot)$ in Assumption~\ref{asmp: heavy_tail}, (c) is due to the smoothness of $b_0(x,\theta)$ in $\theta$ by Assumption~\ref{asmp: heavy_tail}. When $\eta\to0$, we have:
    \begin{align*}
        \eta^{-1}\sum_{k=0}^{\lfloor T/\eta\rfloor-1}\frac{C_f^2\PB(\|H_k(\eta)\|>\varepsilon)^2}{1-C_f\PB(\|H_k(\eta)\|>\varepsilon)}&\lesssim\eta^{-1}\sum_{k=0}^{\lfloor T/\eta\rfloor-1}\PB(\|H_k(\eta)\|>\varepsilon)^2\\
        &\le2\sum_{k=0}^{\lfloor T/\eta\rfloor-1}\eta\frac{C(\bar{\theta}(k\eta))^2}{\varepsilon^{2\alpha}}+\OM(1),
    \end{align*}
    where $\sum_{k=0}^{\lfloor T/\eta\rfloor-1}\eta\frac{C(\bar{\theta}(k\eta))^2}{\varepsilon^{2\alpha}}\to\frac{\int_{0}^TC(\bar{\theta}(t))^2dt}{\varepsilon^{2\alpha}}$.
    }
    Thus, we have:
    \begin{align*}
        0\le-\ln \EB\exp(-N_{\eta}(f))-        \sum_{k=0}^{\lfloor T/\eta\rfloor-1}\left(1-\EB\exp\left(-f(k\eta,H_k(\eta))\right)\right)\le\mathcal{O}(\eta).
    \end{align*}
    Furthermore, we have:
    \begin{align*}
        &\sum_{k=0}^{\lfloor T/\eta\rfloor-1}\left(1-\EB\exp\left(-f(k\eta,H_k(\eta))\right)\right)\\
        =&\sum_{k=0}^{\lfloor T/\eta\rfloor-1}\int_{ x\in\bar{\RB}^d}1-\exp(-f(k\eta,x))\PB(H_k(\eta)\in dx)\notag\\
        \to&\int_{t\in[0,T], x\in\bar{\RB}^d}1-\exp(-f(t,x))\nu(\bar{\theta}(t),dx)dt.\notag
    \end{align*}
    Thus, we have $N_{\eta}\Rightarrow N:=\sum_{i=1}^{\infty}\delta_{t_i,X_i}$, where $N$ is a Poisson point process with measure $\nu(\bar{\theta}(t),dx)dt$. As the summation functional $\psi^{\varepsilon}$ is almost surely continuous, we have:
    \begin{align*}
        \sum_{k=0}^{\lfloor \cdot/\eta\rfloor-1}H_{k}(\eta)\1(\|H_k(\eta)\|>\varepsilon)\Rightarrow\sum_{t_i\le \cdot}X_i\1(\|X_i\|>\varepsilon).
    \end{align*}
    Besides, for any $t\in[0,T]$, we have:
    \begin{align*}
        \EB\sum_{k=0}^{\lfloor t/\eta\rfloor-1}H_{k}(\eta)\1(1\ge\|H_k(\eta)\|>\varepsilon)&=\sum_{k=0}^{\lfloor t/\eta\rfloor-1}\int_{\|x\|\in(\varepsilon,1]} x\PB(H_k(\eta)\in dx)\\
        &\to\int_{s\in[0,t],x\in(\varepsilon,1]}x\nu(\bar{\theta}(s),dx)ds.\notag
    \end{align*}
    Then, we have:
    \begin{align*}
        V_{\eta}^{\varepsilon}(\cdot)\Rightarrow\sum_{t_i\le\cdot}X_i\1(\|X_i\|>\varepsilon)-\int_{s\in[0,\cdot],x\in(\varepsilon,1]}x\nu(\bar{\theta}(s),dx)ds:=V^{\varepsilon}(\cdot).
    \end{align*}
    Thus, by Lemma 20.7 in \citep{ken1999Levy}, there exists an additive process $V(\cdot)$ with characteristic triple $(0,\nu_t,0)$ such that, as $\varepsilon\to0$,
    \begin{align*}
        V^{\varepsilon}(\cdot)\Rightarrow V(\cdot).
    \end{align*}
    Moreover, we have:
    \begin{align*}
        \|V_{\eta}^{\varepsilon}(t)-V^{0}_{\eta}(t)\|=\left\|\sum_{k=0}^{\lfloor t/\eta\rfloor-1}H_k(\eta)\1\left(\left\|H_k(\eta)\right\|\le\varepsilon\right)-\EB\left[ H_k(\eta)\1\left(\left\|H_k(\eta)\right\|\le\varepsilon\right)\right]\right\|.
    \end{align*}
    Thus, for any $\delta>0$, we have:
    \begin{align*}
        \PB\left(\sup_{t\in[0,T]}\|V_{\eta}^{\varepsilon}(t)-V^{0}_{\eta}(t)\|>\delta\right)&\le\frac{1}{\delta^2}\EB\|V_{\eta}^{\varepsilon}(T)-V^{0}_{\eta}(T)\|^2\\
        &=\frac{1}{\delta^2}\sum_{k=0}^{\lfloor T/\eta\rfloor-1}\Var\left(\|H_k(\eta)\|\1\left(\left\|H_k(\eta)\right\|\le\varepsilon\right)\right)\notag\\
        &\le\frac{1}{\delta^2}\sum_{k=0}^{\lfloor T/\eta\rfloor-1}\EB\left(\|H_k(\eta)\|^2\1\left(\left\|H_k(\eta)\right\|\le\varepsilon\right)\right)\notag\\
        &=\frac{1}{\delta^2}\sum_{k=0}^{\lfloor T/\eta\rfloor-1}\eta^{-1}\int_{\|x\|\le\varepsilon}\|x\|^2\mu_{\eta}(\bar{\theta}(k\eta),dx)\notag
    \end{align*}
    where $\mu_{\eta}(\theta,\cdot)=\eta^{-1}\PB\left(\frac{\nabla \ell(\theta,\xii)}{\eta^{\frac{1}{\alpha}}b_1(\eta^{-1})}\in \cdot\right)$. By Assumption~\ref{asmp: heavy_tail}, we have:
    \begin{align*}
        \sum_{k=0}^{\lfloor T/\eta\rfloor-1}\eta^{-1}\int_{\|x\|\le\varepsilon}\|x\|^2\mu_{\eta}(\bar{\theta}(k\eta),dx)=\mathcal{O}\left(\varepsilon^{2-\alpha}\right).
    \end{align*}
    Thus, $\lim_{\varepsilon\to0}\lim\sup_{n}\PB\left(\sup_{t\in[0,T]}\|V_{\eta}^{\varepsilon}(t)-V_{\eta}(t)\|>\delta\right)=0$. By second converging together theorem (Theorem 3.5 in \cite{resnick2007heavy}), we have $V_{\eta}(\cdot)\Rightarrow V(\cdot)$. Besides, as $\mu_{\eta}(\theta,\cdot)$ is uniformly {\rv integrable} by Lemma~\ref{lem: unif_integral} and the expectation is Lipschitz w.r.t. $\theta$ by Lemma~\ref{lem: lip_expectation}. Then, we have:
    \begin{align*}
        \EB\sum_{k=0}^{\lfloor t/\eta\rfloor-1}H_{k}(\eta)\1(\|H_k(\eta)\|>1)\to\int_{s\in[0,t],\|x\|>1}x\nu(\bar{\theta}(s),dx)ds:=\gamma_t.
    \end{align*}
    Thus, we have:
    \begin{align*}
        \eta^{\frac{1}{\alpha}}b_1(\eta^{-1}) \sum_{k=0}^{\lfloor \cdot/\eta\rfloor-1}M(\bar{\theta}(k\eta),\xii_{(k+1)\eta})&=V_{\eta}(\cdot)-\EB\sum_{k=0}^{\lfloor \cdot/\eta\rfloor-1}H_{k}(\eta)\1(\|H_k(\eta)\|>1)\\
        &\Rightarrow V(\cdot)-\gamma_{\cdot}:=L(\cdot),\notag
    \end{align*}
    where $L(\cdot)$ is an additive process with characteristic triple $(0,\nu_t, -\gamma_t)$.
\end{proof}
Thus, combining Lemma~\ref{lem: hessian_approximate},~\ref{lem: error_diff}, and~\ref{lem: additive_converge}, we have:
\begin{align*}
    Z_{\eta}(t)=-\int_0^t\nabla^2\ell(\bar{\theta}(s))Z_{\eta}(s)ds+L_{\eta}(t)+\widetilde{\delta}_{\eta}(t),
\end{align*}
where $\EB\sup_{t\le T}\|\widetilde{\delta}_{\eta}(t)\|=o(1)$. We also denote a variant of $Z_{\eta}(t)$ as:
\begin{align*}
    \widetilde{Z}_{\eta}(t)=-\int_0^t\nabla^2\ell(\bar{\theta}(s))\widetilde{Z}_{\eta}(s)ds+L_{\eta}(t).
\end{align*}
On one side, we prove the $Z_{\eta}(\cdot)$ is almost the same with $\widetilde{Z}_{\eta}(t)$. We denote $\Delta_{\eta}(t):=Z_{\eta}(t)-\widetilde{Z}_{\eta}(t)$ and have:
\begin{align*}
    \Delta_{\eta}(t)=-\int_0^t\nabla^2\ell(\bar{\theta}(s))\Delta_{\eta}(s)ds+\widetilde{\delta}_{\eta}(t).
\end{align*}
Taking norm and expectation both sides, we have:
\begin{align*}
    \EB\|\Delta_{\eta}(t)\|\le H\int_0^t\EB\|\Delta_{\eta}(s)\|ds+\EB\|\widetilde{\delta}_{\eta}(t)\|.
\end{align*}
Multiplying both sides by $\exp(-Ht)$ and integrating over $t$, we have:
\begin{align}
\label{eq: ode_error}
    \exp(-Ht)\int_0^t\EB\|\Delta_{\eta}(s)\|ds\le\int_0^t\exp(-Hs)\EB\|\widetilde{\delta}_{\eta}(s)\|ds.
\end{align}
Thus, we have:
\begin{align*}
    \EB\sup_{0\le t\le T}\|\Delta_{\eta}(t)\|&\le H\int_0^T\EB\|\Delta_{\eta}(s)\|ds+\EB\sup_{0\le t\le T}\|\widetilde{\delta}_{\eta}(t)\|\\
    &\le H\int_o^T\exp(H(T-s)\EB\|\widetilde{\delta}_{\eta}(s)\|ds+\EB\sup_{0\le t\le T}\|\widetilde{\delta}_{\eta}(t)\|\notag\\
    &=o(1).\notag
\end{align*}
On the other side, by Lemma~\ref{lem: ito}, we have:
\begin{align}
    \widetilde{Z}_{\eta}(t)&=\Phi(t)(\widetilde{Z}(0)+\int_0^t\Phi(s)^{-1}dL_{\eta}(t)){\rv ,}\label{eq: Z_eta_t}\\
    Z(t)&=\Phi(t)(Z(0)+\int_0^t\Phi(s)^{-1}dL(t)){\rv .}\label{eq: Z_t}
\end{align}
As $L_{\eta}(\cdot)\Rightarrow L(\cdot)$ in $J_1$ topology, by Lemma~\ref{lem: cadlag_converge}, we have $L_{\eta}(\cdot)$ is tight and weakly converges to $L(\cdot)$ in finite dimensions. Thus, by~\eqref{eq: Z_eta_t} and~\eqref{eq: Z_t}, we have $\widetilde{Z}_{\eta}(\cdot)$ weakly converges to $Z(\cdot)$ in finite dimensions. Then, we only need to prove $\widetilde{Z}_{\eta}(\cdot)$ is tight. In fact, by~\eqref{eq: Z_eta_t}, we have:
\begin{align}
\label{eq: tight_z_1}
    \sup_{0\le t\le T}\|\widetilde{Z}_{\eta}(t)\|&=\OM(\sup_{0\le t\le T}\|L_{\eta}(t)\|),\\
\label{eq: tight_z_2}
    \|\widetilde{Z}_{\eta}(t_1)-\widetilde{Z}_{\eta}(t_2)\|&=\OM(\|\Phi(t_1)-\Phi(t_2)\|)+\OM(\|L_{\eta}(t_1)-L_{\eta}(t_2)\|).
\end{align}
As $\nabla^2\ell(\bar{\theta}(t))$ is continuous, then $\Phi(\cdot)$ is continuous by \cite{baake2011peano}. By Lemma~\ref{lem: cadlag_tight}, as $L_{\eta}(\cdot)$ is tight, we have $Z_{\eta}(\cdot)$ is also tight. Thus, $\widetilde{Z}_{\eta}(\cdot)\Rightarrow Z(\cdot)$. Moreover, as $\EB\sup_{0\le t\le T}\|\Delta_{\eta}(t)\|=o(1)$, we have $Z_{\eta}(\cdot)\Rightarrow Z(\cdot)$.
\subsection{Proof of Theorem~\ref{thm: asymp_decay}} 
{\rv In the proof, we first assume the sequence $\eta_n^{\frac{1}{\alpha}-1}b_1(\eta_n^{-1})(\theta_n-\theta^*)$ is tight. We will verify  the tightness in next section.} Before the proof, we give a lemma characterizing the convergence rate of $\theta_n-\theta^*$:
\begin{lemma}[Theorem 3 in \cite{wang2021convergence}]
\label{lem: theta_converge}
    For any $p\in[1,\alpha)$, we have:
    \begin{align*}
        \EB\|\theta_n-\theta^*\|^p=\OM(\eta_n^{p-1}).
    \end{align*}
\end{lemma}
We denote $X_n:=\eta_n^{\frac{1}{\alpha}-1}b_1(\eta_n^{-1})(\theta_n-\theta^*)$ as the scaled error. By the SGD update rule in~\eqref{eq: sgd}, we have:
\begin{align*}
    X_{n+1}&=\eta_{n+1}^{\frac{1}{\alpha}-1}b_1(\eta_{n+1}^{-1})(\theta_n-\theta^*-\eta_n\nabla \ell(\theta_n,\xii_{n+1}))\\
    &=\eta_{n+1}^{\frac{1}{\alpha}-1}b_1(\eta_{n+1}^{-1})(\theta_n-\theta^*-\eta_n\nabla \ell(\theta_n)-\eta_n M(\theta_n,\xii_{n+1})),\notag
\end{align*}
where $M(\theta_n,\xii_n)=\nabla\ell(\theta_n,\xii_{n+1})-\nabla\ell(\theta_n)$. By Assumption~\ref{asmp: q_expansion} and first order condition $\nabla\ell(\theta^*)=0$, we have:
\begin{align*}
    \nabla \ell(\theta_n)&=\nabla \ell(\theta_n)-\nabla \ell(\theta^*)\\
    &=\nabla^2\ell(\theta^*)(\theta_n-\theta^*)+R_n,\notag
\end{align*}
where $\|R_n\|=O(\|\theta_n-\theta^*\|^q)$. Thus, we have:
\begin{align*}
    X_{n+1}=\frac{\eta_{n+1}^{\frac{1}{\alpha}-1}b_1(\eta_{n+1}^{-1})}{\eta_{n}^{\frac{1}{\alpha}-1}b_1(\eta_{n}^{-1})}\left((I-\eta_n \nabla^2\ell(\theta^*))X_n+\eta_n^{\frac{1}{\alpha}}b_1(\eta_n^{-1})M(\theta_n,\xii_{n+1})+\eta_n^{\frac{1}{\alpha}}b_1(\eta_n^{-1})R_n\right).
\end{align*}
By Lemma~\ref{lem: eta_expansion}, we have:
\begin{align}
\label{eq: Z}
    X_{n+1}=(1+\OM(n^{-1}))\left((I-\eta_n \nabla^2\ell(\theta^*))X_n+\eta_n^{\frac{1}{\alpha}}b_1(\eta_n^{-1})M(\theta_n,\xii_{n+1})+\eta_n^{\frac{1}{\alpha}}b_1(\eta_n^{-1})R_n\right).
\end{align}
Thus, we construct a similar sequence $\{\widetilde{X}_n\}_{n\ge1}$ satisfying:
\begin{align}
\label{eq: tilde_Z}
    \widetilde{X}_{n+1}=(I-\eta_n \nabla^2\ell(\theta^*))\widetilde{X}_n+\eta_n^{\frac{1}{\alpha}}b_1(\eta_n^{-1})M(\theta^*,\xii_{n+1}).
\end{align}
The following lemma makes sure weakly convergence of $\{\widetilde{X}_n\}_{n\ge1}$ is the same with weakly convergence of $\{X_n\}_{n\ge1}$.

\begin{lemma}
\label{lem: X_equi}
    If the sequence $\widetilde{X}_n$ in~\eqref{eq: tilde_Z} weakly converges to $Z$, then $X_n$ in~\eqref{eq: Z} also weakly converges to $Z$.
\end{lemma}
\begin{proof}
    We denote $\Delta_n=X_n-\widetilde{X}_n$, it satisfies:
    \begin{align}
    \label{eq: Delta_n}
        \Delta_{n+1} = &(I-\eta_n\nabla^2\ell(\theta^*))\Delta_n+\eta_n^{\frac{1}{\alpha}}b_1(\eta_n^{-1})\left(M(\theta_n,\xii_{n+1})-M(\theta^*,\xii_{n+1})\right)+\eta_n^{\frac{1}{\alpha}}b_1(\eta_n^{-1})R_n\\
        &+\OM(n^{-1})\left((I-\eta_n \nabla^2\ell(\theta^*))X_n+\eta_n^{\frac{1}{\alpha}}b_1(\eta_n^{-1})M(\theta_n,\xii_{n+1})+\eta_n^{\frac{1}{\alpha}}b_1(\eta_n^{-1})R_n\right)\notag\\
        =&(I-\eta_n\nabla^2\ell(\theta^*))\Delta_n+\eta_n^{\frac{1}{\alpha}}b_1(\eta_n^{-1})\left(M(\theta_n,\xii_{n+1})-M(\theta^*,\xii_{n+1})\right)+\eta_n^{\frac{1}{\alpha}}b_1(\eta_n^{-1})R_n\notag\\
        &+\OM(n^{-1})\eta_{n}^{\frac{1}{\alpha}-1}b_1(\eta_{n}^{-1})(\theta_{n+1}-\theta^*).\notag
    \end{align}
    Thus, we decompose $\Delta_n=\Delta_n^{(1)}+\Delta_n^{(2)}+\Delta_n^{(3)}$, where:
    \begin{align*}
        \Delta_{n+1}^{(1)} = &(I-\eta_n\nabla^2\ell(\theta^*))\Delta_n^{(1)}+\eta_n^{\frac{1}{\alpha}}b_1(\eta_n^{-1})\left(M(\theta_n,\xii_{n+1})-M(\theta^*,\xii_{n+1})\right),\\
        \Delta_{n+1}^{(2)} = &(I-\eta_n\nabla^2\ell(\theta^*))\Delta_n^{(2)}+\eta_n^{\frac{1}{\alpha}}b_1(\eta_n^{-1})R_n,\\
        \Delta_{n+1}^{(3)} = &(I-\eta_n\nabla^2\ell(\theta^*))\Delta_n^{(3)}+\OM(n^{-1})\eta_{n}^{\frac{1}{\alpha}-1}b_1(\eta_{n}^{-1})(\theta_{n+1}-\theta^*).
    \end{align*}
    For the $\Delta_n^{(1)}$, noticing that $\EB[M(\theta_n,\xii_{n+1})-M(\theta^*,\xii_{n+1})|\FM_n]=0$ and $\xii_{n+1}$ is independent with $\FM_n$. Thus, for any $\frac{2\alpha}{\alpha+1}<p<\alpha$ we have:
    \begin{align*}
        \EB\|\Delta_{n+1}^{(1)}\|^p&\le \|I-\eta_n\nabla^2\ell(\theta^*)\|^p\EB\|\Delta_n^{(1)}\|^p+\OM(\eta_n^{\frac{p}{\alpha}}b_1^p(\eta_n^{-1})\EB\|M(\theta_n,\xii_{n+1})-M(\theta^*,\xii_{n+1})\|^p)\\
        &\le (1-L\eta_n)\EB\|\Delta_n^{(1)}\|^p+\OM(\eta_n^{\frac{p}{\alpha}}b_1^p(\eta_n^{-1})\EB\|\theta_n-\theta^*\|^p).\notag
    \end{align*}
    By Lemma~\ref{lem: theta_converge}, we have $\EB\|\theta_{n+1}-\theta^*\|^p=\OM(\eta_{n+1}^{p-1})$. As $\frac{2\alpha}{\alpha+1}<p<\alpha$, we have:
    \begin{align*}
        \EB\|\Delta_{n+1}^{(1)}\|^p&\le (1-L\eta_n)\EB\|\Delta_n^{(1)}\|^p+\OM(\eta_n^{\frac{p}{\alpha}+p-1}b_1^p(\eta_n^{-1}))\\
        &\le(1-L\eta_n)\EB\|\Delta_n^{(1)}\|^p+\eta_n o\left(\eta_n^{\frac{\frac{p}{\alpha}+p-2}{2}}\right).\notag
    \end{align*}
    Thus, by Lemma~\ref{lem: recursion}, we have $\lim_{n\to\infty}\EB\|\Delta_{n}^{(1)}\|^p=0$. Thus, $\Delta_n^{(1)}$ converges to 0 in probability.

    \noindent For the $\Delta_n^{(2)}$, as $R_n=O(\|\theta_n-\theta^*\|^q)$, taking norm and expectation both sides, we have:
    \begin{align*}
        \EB\|\Delta_{n+1}^{(2)}\|\le(1-\sigma_{\min}(\nabla^2\ell(\theta^*))\eta_n)\EB\|\Delta_n^{(2)}\|+\OM(\eta_n^{\frac{1}{\alpha}}b_1(\eta_n^{-1})\EB\|\theta_n-\theta^*\|^q).
    \end{align*}
    By Lemma~\ref{lem: theta_converge}, we have $\EB\|\theta_{n+1}-\theta^*\|^q=\OM(\eta_{n+1}^{q-1})$ as $q<\alpha$. Then,
    \begin{align*}
        \EB\|\Delta_{n+1}^{(2)}\|\le(1-\sigma_{\min}(\nabla^2\ell(\theta^*))\eta_n)\EB\|\Delta_n^{(2)}\|+\OM(\eta_n^{\frac{1}{\alpha}}b_1(\eta_n^{-1})\eta_n^{q-1}).
    \end{align*}
    Moreover, as $q>2-\frac{1}{\alpha}$, we have:
    \begin{align*}
        \EB\|\Delta_{n+1}^{(2)}\|\le(1-\sigma_{\min}(\nabla^2\ell(\theta^*))\eta_n)\EB\|\Delta_n^{(2)}\|+\eta_n o\left(\eta_n^{\frac{\frac{1}{\alpha}+q-2}{2}}\right).
    \end{align*}
    Thus, by Lemma~\ref{lem: recursion}, we have $\lim_{n\to\infty}\EB\|\Delta_{n}^{(2)}\|=0$. Thus, $\Delta_n^{(2)}$ converges to 0 in probability.
    
    \noindent By Lemma~\ref{lem: theta_converge}, we have $\EB\|\theta_{n+1}-\theta^*\|\le(\EB\|\theta_{n+1}-\theta^*\|^p)^{\frac{1}{p}} =\OM(\eta_{n+1}^{1-\frac{1}{p}})$, where $p\in[1,\alpha)$. Thus, taking norm and expectation both sides in~\eqref{eq: Delta_n}, we have:
    \begin{align*}
        \EB\|\Delta_{n+1}^{(3)}\|\le (1-\sigma_{\min}(\nabla^2\ell(\theta^*))\eta_n)\EB\|\Delta_n^{(3)}\|+\OM\left(\eta_n^{\frac{1}{\rho}+\frac{1}{\alpha}-\frac{1}{p}}b_1(\eta_n^{-1})\right).
    \end{align*}
    As $\rho<1$, we have $\frac{1}{\rho}+\frac{1}{\alpha}-\frac{1}{p}-1>0$ as long as $\alpha>p>\frac{1}{\frac{1}{\rho}+\frac{1}{\alpha}-1}$. Thus, we have:
    \begin{align*}
        \EB\|\Delta_{n+1}^{(3)}\|\le (1-\sigma_{\min}(\nabla^2\ell(\theta^*))\eta_n)\EB\|\Delta_n^{(3)}\|+\eta_n o\left(\eta_n^{\frac{\frac{1}{\rho}+\frac{1}{\alpha}-\frac{1}{p}-1}{2}}\right).
    \end{align*}
    Then, by Lemma~\ref{lem: recursion}, we have $\lim_{n\to\infty}\EB\|\Delta_{n}^{(3)}\|=0$. Thus, $\Delta_n^{(3)}$ converges to 0 in probability. Finally, as $\Delta_n^{(1)}$, $\Delta_n^{(2)}$, and $\Delta_n^{(3)}$ all converges to 0 in probability, we have $\Delta_n$ converges to 0 in probability.
\end{proof}
Thus, we only need to study the convergence of $\{\widetilde{X}_n\}_{n\ge 1}$. In our conjecture, $X_n$ will weakly converge to a stationary distribution of s.d.e. $dZ_t = -\nabla^2\ell(\theta^*)Z_t dt+dL_t$, where $L_t$ is a $\cadlag$ stochastic process. To translate the discrete sequence to a continuous time, we introduce a new $\cadlag$ stochastic process $(Y_t)_{t\ge0}$ such that:
\begin{align}
    \label{eq: Y_t}
    Y_{t} = \widetilde{X}_n-(t-s_{n-1})\nabla^2\ell(\theta^*)\widetilde{X}_n, \text{ for $t\in[s_{n-1},s_n)$},
\end{align}
where $s_n=\sum_{k=0}^{n}\eta_k$ and $s_{-1}:=0$. The motivation is to interpolate the points $\{\widetilde{X}_n\}_{n\ge1}$ with a jump in Figure~\ref{fig: sto_process}. And we denote a family of stochastic process $(Y^{u}_{\cdot})_{u\ge0}$ where $Y_t^u:=Y_{t+u}$. 
\begin{figure}[htbp!]
    \centering
    \includegraphics[width=0.8\textwidth]{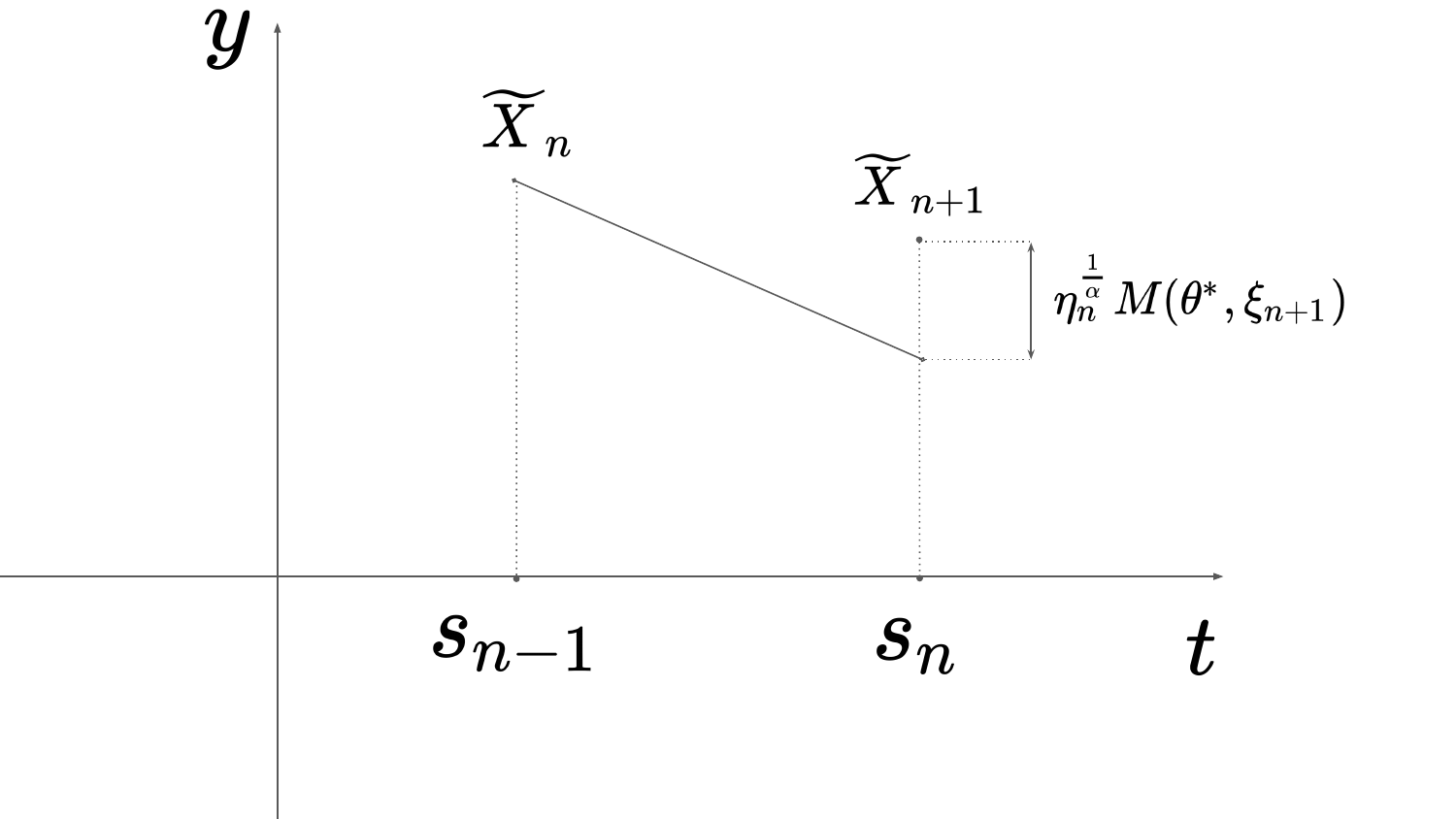}
    \caption{Illustration of stochastic process $(Y_t)_{t\ge0}$.}
    \label{fig: sto_process}
\end{figure}

\noindent For any $t\in[s_{n-1},s_n)$, we can further decompose the $Y_t$ into:
\begin{align*}
    Y_t = \widetilde{X}_0+H_t+M_t,
\end{align*}
where $H_t=H_{s_{n-1}}-(t-s_{n-1})\nabla^2\ell(\theta^*)\widetilde{X}_n$ is an iterative representation of $H_{\cdot}$ ($H_0=0$), and $M_t= \sum_{k=0}^{n-1}\eta_k^{\frac{1}{\alpha}}b_1(\eta_k^{-1})M(\theta^*,\xii_{k+1})$. 
And we denote a family of stochastic process $(Y^{u}_{\cdot})_{u\ge0}$ where $Y_t^u:=Y_{t+u}$. For any $u,t$, we have:
\begin{align*}
    Y_t^u-Y_0^u = H_{t+u}-H_u+M_{t+u}-M_u:=H_t^u-H_0^u+M_t^u-M_0^u.
\end{align*}
Then, we study the approximation of $H_t^u-H_0^u$ and $M_t^u-M_0^u$ separately.
\begin{lemma}
\label{lem: H_Y}
    For any $T$, we have:
    \begin{align*}
        \lim_{u\to\infty}\EB\sup_{t\le T}\left\|H_{t}^u-H_0^u+\int_0^T\nabla^2\ell(\theta^*)Y_s^uds\right\|=0.
    \end{align*}
\end{lemma}
\begin{proof}
    We denote $s_n^u=s_n-u$. Indeed, we have:
    \begin{align*}
        H_t^u-H_0^u+\int_{0}^t\nabla^2\ell(\theta^*)Y_s^uds=- \int_{u}^{u+t}\nabla h(\theta^*)\left(\widetilde{X}_s-Y_s\right)ds,
    \end{align*}
    where $\widetilde{X}_s = \widetilde{X}_n$ when $s\in[s_{n-1}^u,s^u_{n})$. Thus, we have:
    \begin{align*}
        &\sup_{t\le T}\left\|H_t^u-H_0^u+\int_{0}^t\nabla^2\ell(\theta^*)Y_s^uds\right\|\\
        \le&\|\nabla h(\theta^*)\|\int_{u}^{u+T}\left\|\widetilde{X}_s-Y_s\right\|ds\notag\\
        =&\|\nabla^2\ell(\theta^*)\|\left(\sum_{m\le i\le n-1}\int_{s_i}^{s_{i+1}}\left\|\widetilde{X}_s-Y_s\right\|ds+\int_{u}^{s_m}\left\|\widetilde{X}_s-Y_s\right\|ds+\int_{s_{n-1}}^{u+t}\left\|\widetilde{X}_s-Y_s\right\|ds\right)\notag
    \end{align*}
    where $m,n$ satisfies $u\in[s_{m-1},s_m)$ and $u+T\in[s_{n-1},s_n)$. By definition of $Y$, for $s\in[s_{i-1},s_i)$, we have:
    \begin{align*}
        \widetilde{X}_{s}-Y_s= (s-s_{i-1})\nabla^2\ell(\theta^*)\widetilde{X}_{i}.
    \end{align*}
    Thus, we have:
    \begin{align*}
        \EB\sup_{t\le T}\left\|H_t^u-\int_{0}^t\nabla h(\theta^*)Y_s^uds\right\|=\OM\left(\sum_{i:0<s_i^u\le T}\eta_i^2\EB\|\widetilde{X}_i\|\right).
    \end{align*}
    By~\eqref{eq: tilde_Z}, we have:
    \begin{align*}
        \eta_{n+1}^{2}\EB\|\widetilde{X}_{n+1}\|\le(1-L\eta_n)\eta_n^{2}\EB\|\widetilde{X}_n\|+\OM\left(\eta_n^{\frac{1}{\alpha}+2}b_1(\eta_n^{-1})\right).
    \end{align*}
    Thus, we have:
    \begin{align*}
        \eta_n^2\EB\|\widetilde{X}_{n}\|\le o\left(\eta_n^{\frac{1}{2\alpha}+1}\right).
    \end{align*}
    Then, we have:
    \begin{align*}
        \EB\sup_{t\le T}\left\|H_t^u-H_0^u-\int_{0}^t\nabla^2\ell(\theta^*)Y_s^uds\right\|=o\left(\sum_{i:0<s_i^u\le T}\eta_i^{\frac{1}{2\alpha}+1}\right)=o\left(\eta_m^{\frac{1}{2\alpha}}\right),
    \end{align*}
    where the last step we use the fact that $\sum_{i:0<s_i^u\le T}\eta_i=T$. Letting $u\to\infty$, then $m\to\infty$, and our result is concluded.
\end{proof}
By Lemma~\ref{lem: H_Y}, we can study the following stochastic process directly:
\begin{align*}
    \widetilde{Y}_t^u= \widetilde{Y}_0^u-\int_0^t\nabla^2\ell(\theta^*) \widetilde{Y}_s^uds+M_t^u-M_0^u.
\end{align*}
We also have $\widetilde{Y}_t^u=\widetilde{Y}_{t+u}$. In fact, we have:
\begin{align*}
    Y_t^u-\widetilde{Y}_t^u &= -\int_0^t\nabla^2\ell(\theta^*)(Y_s^u-\widetilde{Y}_s^u)ds+H_t^u-H_0^u+\int_0^t\nabla^2\ell(\theta^*)Y_s^uds\\
    &:=-\int_0^t\nabla^2\ell(\theta^*)(Y_s^u-\widetilde{Y}_s^u)ds+\delta_t^u.\notag
\end{align*}
Taking norm and expectation both sides, we have:
\begin{align*}
    \EB\|Y_t^u-\widetilde{Y}_t^u\|\le C\int_0^t\EB\|Y_s^u-\widetilde{Y}_s^u\|ds+\|\delta_t^u\|.
\end{align*}
By a similar argument in~\eqref{eq: ode_error} and Lemma~\ref{lem: H_Y}, we have $Y_t^u-\widetilde{Y}_t^u$ converges to 0 in probability. Next lemma provides the weak convergence of $M_t^u-M_0^u$.
\begin{lemma}
        \label{lem: Mu_L\'evy}
    $(M^u_{\cdot}-M_0^u)_{u\ge0}$ weakly converges to a L\'evy process $L_t$ with characteristics $(0,\nu(\theta^*,\cdot),-\gamma)$ satisfying:
    \begin{align*}
        &\gamma = \int_{\|x\|>1}x\nu(\theta^*,dx).
    \end{align*}
\end{lemma}
\begin{proof}
    Indeed, we can write $M^u_t-M_0^u$ as:
    \begin{align*}
        M^u_t-M_0^u = \sum_{n: 0 <s_{n}^u\le t} \eta_n^{\frac{1}{\alpha}}b_1(\eta_n^{-1})M(\theta^*,\xii_{n+1}),
    \end{align*}
    where $M(\theta^*,\xii_{n})$ is a sequence of i.i.d. random variables and $s_n^u:=s_n-u$. For any fixed $u$, $M^u_{\cdot}$ is a martingale. We restrict our analysis on $[0,T]$. We define a point process as:
    \begin{align*}
        N_u(dt,dx)=\sum_{s_n^u>0}\delta_{s_n^u,\eta_n^{\frac{1}{\alpha}}b_1(\eta_n^{-1})\nabla\ell(\theta^*,\xii_{n+1})}(dt,dx).
    \end{align*}
    For any bounded function $f$ with compact support on $[0,T]\times \RB^d$, by a similar argument in the proof of Lemma~\ref{lem: additive_converge}, the laplacian functional of $N_u$ is:
    \begin{align*}
        -\log \EB\exp(-N_u(f))&{\rv =}\sum_{0<s_n^u\le T}\eta_n\int\left(1-\exp\left(-f(s_n^u,x)\right)\right)\mu_n(\theta^*,dx) {\rv +o(1)},\\
        &=\sum_{0<s_n^u\le T}(s_n^u-s_{n-1}^u)\int\left(1-\exp\left(-f(s_n^u,x)\right)\right)\mu_n(\theta^*,dx){\rv +o(1)}\notag\\
        &\rightarrow \int_{0}^T\int\left(1-\exp\left(-f(t,x)\right)\right)\nu(\theta^*,dx)dt.\notag
    \end{align*}
    where $\mu_n(\theta^*,dx)=\eta_n^{-1}\PB(\eta_n^{\frac{1}{\alpha}}b_1(\eta_n^{-1})\nabla\ell(\theta^*,\xii)\in dx)$.
    Thus, $N_u(\cdot)\Rightarrow N$, where $N$ is a Poisson point process with measure $\nu(\theta^*, dx)dt$. As we have:
    \begin{align*}
        M_T^u-M_0^u=\int_{t\le T}xN_u(dt,dx)-\int_{t\le T}x\EB N_u(dt,dx),
    \end{align*}
    and summation functional is continuous on $J_1$ topology. We conclude $M^u_{\cdot}-M_0^u\Rightarrow L_{\cdot}$, where $L_{\cdot}$ is a L\'evy process with characteristics $(0, \nu(\theta^*,\cdot),-\gamma)$, where
    \begin{align*}
        \gamma = \int_{\|x\|>1}x\mu(\theta^*,dx).
    \end{align*}
\end{proof}
By Lemma~\ref{lem: ito}, we have:
\begin{align*}
    \widetilde{Y}_t^u = \exp(-\nabla^2\ell(\theta^*)t)\left(\widetilde{Y}_0^u+\int_0^t\exp(\nabla^2\ell(\theta^*)s)d(M_s^u-M_0^u)\right).
\end{align*}
We also denote:
\begin{align*}
    Z_t=\exp(-\nabla^2\ell(\theta^*)t)\left(Z_0+\int_0^t\exp(\nabla^2\ell(\theta^*)s)dL_s\right).
\end{align*}
By a similar argument in~\eqref{eq: tight_z_1} and~\eqref{eq: tight_z_2}, given $\widetilde{Y}_0^u=Z_0=x$, we have $\widetilde{Y}_{\cdot}^u\Rightarrow Z_{\cdot}$.
As $\{X_n\}$ is tight, then $\{\widetilde{X}_n\}$ is tight, and then $\{Y_t\}_{t\ge1}$ is tight. By Prohorov’s Theorem, we can consider a closure point $\nu$ of $\widetilde{Y}_u$ such that there exists a sequence $u(n)$ satisfying $\PB^{-1}(\widetilde{Y}_{u(n)})\Rightarrow\nu$. Furthermore, fix a $t$, consider a closure point $\nu_t$ of $\widetilde{Y}_{u(n)-t}$ such that a subsequence $w(n)$ of $u(n)$ satisfying $\widetilde{Y}_{w(n)-t}\Rightarrow \nu_t$. And we denote $\mu$ as the stationary distribution of $Z_t$. Then, for any continuous and bounded function $\psi$, we have:
\begin{align*}
    |\EB_{X\sim\nu}\psi(X)-\EB_{X\sim\mu}\psi(X)|\le|\EB_{X\sim\nu}\psi(X)-\EB\psi(\widetilde{Y}_{w(n)})|+|\EB\psi(\widetilde{Y}_{w(n)})-\EB_{X\sim\mu}\psi(X)|.
\end{align*}
The first term on RHS is arbitrary small as $\PB^{-1}(\widetilde{Y}_{w(n)})\Rightarrow\nu$. By $\mu$ is the stationary distribution of $Z_t$, the second term can be written as:
\begin{align*}
    &|\EB\psi(\widetilde{Y}_{w(n)})-\EB_{X\sim\mu}\psi(X)|\\
    =&|\EB\psi(\widetilde{Y}^{w(n)-t}_t)-\EB_{Z_0\sim\mu}\psi(Z_t)|\notag\\
    \le&|\EB_{\widetilde{Y}^{w(n)-t}_0\sim\nu_t}\psi(\widetilde{Y}^{w(n)-t}_t)-\EB_{Z_0\sim\nu_t}\psi(Z_t)|+|\EB_{Z_0\sim\nu_t}\psi(Z_t)-\EB_{Z_0\sim\mu}\psi(Z_t)|.\notag
\end{align*}
As given $\widetilde{Y}_0^u=Z_0$, we have $\widetilde{Y}_{\cdot}^u\Rightarrow Z_{\cdot}$. We also have $Z_t$ is fast mixing \cite{lee2012exponential}. Thus, let $n,t\to\infty$, we have $\nu=\mu$. Thus, $X_n\Rightarrow\mu$.

\subsection{\rv Proof of Tightness of sequence $\eta_n^{\frac{1}{\alpha}-1}b_1(\eta_n^{-1})(\theta_n-\theta^*)$}
\label{sec: tight}
By Lemma~\ref{lem: X_equi}, we study the following recursion directly.
\begin{align}
\label{eq: X_recur}
    X_{n+1}=(I-\eta_n \nabla^2\ell(\theta^*))X_n+\eta_n^{\frac{1}{\alpha}}b_1(\eta_n^{-1})M(\theta^*,\xii_{n+1}).
\end{align}
Consider the orthogonal transformation $I-\eta_n \nabla^2\ell(\theta^*) = \Gamma (I-\eta\Sigma)\Gamma^{-1}$, where $\Sigma$ is a diagonal matrix and $\Gamma$ is the orthogonal matrix. Thus, we can write~\eqref{eq: X_recur} as:
\begin{align*}
    \Gamma^{-1}X_{n+1}=(I-\eta_n\Sigma)\Gamma^{-1}X_n+\eta_n^{\frac{1}{\alpha}}b_1(\eta_n^{-1})\Gamma^{-1}M(\theta^*,\xii_{n+1})
\end{align*}
Since $\Gamma^{-1}$ is the orthogonal matrix, the asymptotic Radon measure of $\Gamma^{-1}M(\theta^*,\xii_{n+1})$ will becomes:
    \begin{align*}
       x\PB\left(\frac{\Gamma^{-1}\nabla\ell(\theta^*,\xii)}{x^{\frac{1}{\alpha}}b_1(x)}\in\cdot\right)\overset{v}{\to}\nu_{\Gamma}(\theta^*,\cdot),\text{ $x\to+\infty$},
    \end{align*}
    where $\nu_{\Gamma}(\theta,A)=\nu(\theta,\Gamma A)$ for any Borel set $A$.
    Thus, for any $C>0$, we have:
    \begin{align*}
        \left\{\left\|X_{n}\right\|>C\right\}=\left\{\left\|\Gamma^{-1}X_{n}\right\|>C\right\}\subseteq\cup_{k=1}^d\left\{\left|(\Gamma^{-1}X_n)^{(k)}\right|>\frac{C}{d}\right\},
    \end{align*}
    where $x^{(k)}$ is the $k$-th entry of $x$. Thus, we only need to prove the recursion~\eqref{eq: X_recur} is tight in one dimension case. And also, we could maintain the slowly varying function by Lemma~\ref{lem: dominate_sl}. To ease the notation, we use $M_n$ to replace with $M(\theta^*,\xii_{n+1})$.

\begin{lemma}
\label{lem: sym_tight}
    Suppose $\{\varepsilon_n\}$ is an i.i.d. sequence with symmetric distributions satisfying $\PB(\varepsilon_i>x)=b_0(x)x^{-\alpha}$, then the following sequence is tight:
    \begin{align*}
        \widetilde{X}_{n+1}=(1-\eta_n)\widetilde{X}_n+\eta_n^{\frac{1}{\alpha}} b_1(\eta_n^{-1})\varepsilon_n,
    \end{align*}
    where $b_1(t)$ satisfies $\PB(|\varepsilon_n|>t^{\frac{1}{\alpha}}b_1(t)^{-1})=t^{-1}$.
\end{lemma}
\begin{proof}
The expression of $\widetilde{X}_n$ is:
\begin{align*}
    \widetilde{X}_n = \frac{\sum_{i=1}^{n-1}\prod_{j=1}^i (1-\eta_j\sigma)^{-1}\eta_i^{\frac{1}{\alpha}}b_1(\eta_i^{-1})\varepsilon_i}{\prod_{i=1}^{n-1}(1-\eta_i\sigma)^{-1}}:=\sum_{i=1}^{n-1}\omega_{i,n-1}\eta_i^{\frac{1}{\alpha}}b_1(\eta_i^{-1})\varepsilon_i.
\end{align*}
Then, the tail of $\widetilde{X}_n$ satisfies:
\begin{align*}
    \PB(|\widetilde{X}_n|>t)\le7t\int_0^{1/t}\left(1-\text{Re}(\phi_{\widetilde{X}_n}(u))\right)du,
\end{align*}
where $\phi_X(t)$ is the characteristic function of $X$ {\rv and this inequality follows by 3.2.a in \cite{lin2010probability}}. Moreover, we have:
\begin{align*}
    \log \phi_{\widetilde{X}_n}(t) = \sum_{i=1}^{n-1}\log\phi_{\varepsilon}(\omega_{i,n-1}\eta_i^{\frac{1}{\alpha}}b_1(\eta_i^{-1})t).
\end{align*}
By symmetry of $\varepsilon$, we have:
\begin{align*}
    \PB(|\widetilde{X}_n|>t)&\le\sum_{i=1}^{n-1}7t\int_0^{1/t}-\log\phi_{\varepsilon}(\omega_{i,n-1}\eta_i^{\frac{1}{\alpha}}b_1(\eta_i^{-1})u)du\\
    &\le\sum_{i=1}^{n-1}7Ct\int_0^{1/t}(\omega_{i,n-1}\eta_i^{\frac{1}{\alpha}}b_1(\eta_i^{-1})u)^\alpha b_0\left(\frac{1}{\omega_{i,n}\eta_i^{\frac{1}{\alpha}}b_1(\eta_i^{-1})u}\right)du\notag\\
    &=\sum_{i=1}^{n-1}\frac{7Ct}{\omega_{i,n-1}\eta_i^{\frac{1}{\alpha}}b_1(\eta_i^{-1})}\int_{\frac{t}{\omega_{i,n-1}\eta_i^{\frac{1}{\alpha}}b_1(\eta_i^{-1})}}^{\infty}x^{-(\alpha+2)}b_0(x)dx,\notag
\end{align*}
where in the second inequality, we use the Pitman's Tauberian theorem. By  Karamata’s theorem, we have:
\begin{align}
    \label{eqn: potter}
    \PB(|\widetilde{X}_n|>t)&\le\sum_{i=1}^n 7C\left(\frac{\omega_{i,n-1}\eta_i^{\frac{1}{\alpha}}b_1(\eta_i^{-1})}{t}\right)^{\alpha}b_0\left(\frac{t}{\omega_{i,n-1}\eta_i^{\frac{1}{\alpha}}b_1(\eta_i^{-1})}\right)\\
    &=\frac{7C}{t^\alpha}\sum_{i=1}^n \omega_{i,n-1}^\alpha\eta_i\frac{b_0\left(\frac{t}{\omega_{i,n-1}\eta_i^{\frac{1}{\alpha}}b_1(\eta_i^{-1})}\right)}{b_0\left(\frac{1}{\eta_i^{\frac{1}{\alpha}}b_1(\eta_i^{-1})}\right)}\notag\\
    &\le\frac{7CA}{t^{\alpha-\delta}}\sum_{i=1}^n\omega_{i,n-1}^{\alpha-\delta}\eta_i,\notag
\end{align}
where the last inequality is due to Potter bounds. Thus, if $\delta<\alpha$, by Stolze's theorem, $\sum_{i=1}^{n-1}\omega_{i,n-1}^{\alpha-\delta}\eta_i=\OM(1)$. Then we have $\widetilde{X}_n$ is tight.
\end{proof}
We decompose the random variable $M_n$ in \eqref{eq: X_recur} into $M_n=M_n^+-M_n^-$, where the tails are:
\begin{align*}
    &\PB(M_n^+>x)=b_0^+(x)x^{-\alpha},\\
    &\PB(M_n^->x)=b_0^-(x)x^{-\alpha}.
\end{align*}
We further assume there exists $x_0$, such that $b_0^+(x)\ge b_0^-(x)$ for any $x\ge x_0$. Then we chose $b_1^+(x)$ satisfying $\PB(M_n^+>x^{\frac{1}{\alpha}}b_1^+(x))=x^{-1}$ and $b_1^-(x)$ satisfying $\PB(M_n^->x^{\frac{1}{\alpha}}b_1^-(x))=x^{-1}$, where $b_1^+(x)\le b_1^-(x)$ as $b_0^+(x)\ge b_0^-(x)$. Consider the following recursions:
\begin{align*}
    &Y_{n+1}=(1-\eta_n)Y_n+\eta_n^{\frac{1}{\alpha}}b_1^+(\eta_n^{-1}))\widetilde{M}_n^+=\sum_{i=1}^{n}\omega_{i,n}\eta_i^{\frac{1}{\alpha}}b_1^+(\eta_i^{-1})\widetilde{M}_n^+,\\
    &Z_{n+1}=(1-\eta_n)Z_n+\eta_n^{\frac{1}{\alpha}}b_1^-(\eta_n^{-1}))\widetilde{M}_i^-=\sum_{i=1}^{n}\omega_{i,n}\eta_i^{\frac{1}{\alpha}}b_1^-(\eta_i^{-1})\widetilde{M}_i^-,
\end{align*}
where $\omega_{i,n}=\prod_{k=i+1}^n(1-\eta_k)$, $\widetilde{M}_n^+$ is the independent symmetrized copy of $M_n^+$ and $\widetilde{M}_n^-$ is the independent symmetrized copy of $M_n^-$. The tail probabilities of $\widetilde{M}_n^+$ and $\widetilde{M}_n^-$ satisfy:
\begin{align*}
\PB(\widetilde{M}_n^+>x)=\PB(\widetilde{M}_n^+<-x)=\frac{\PB(M_n^+>x)}{2\PB(M_n>0)},\\
\PB(\widetilde{M}_n^->x)=\PB(\widetilde{M}_n^-<-x)=\frac{\PB(M_n^->x)}{2\PB(M_n<0)}.
\end{align*}
By Lemma~\ref{lem: sym_tight} and Theorem~\ref{thm: asymp_decay}, we have $Y_n$ and $Z_n$ weakly converges. It is worth noticing that we can replace $b_1^-(\eta_i^{-1})$ with $b_1^+(\eta_i^{-1})$ in $Z_n$, where the weakly converging still holds as the tightness of $Z_n$ implies the tightness of the case when  $b_1^-$ is replaced by $b_1^+$. By Lemma~\ref{lem: infinitesimal}, there exists a non-increasing $H^+(u)$ and $H^-(u)$, at every continuity point of $H^+(u)$ and $H^-(u)$, for $u>0$, we have:
\begin{align*}
    \lim_{n\to+\infty}\sum_{i=1}^n\PB(\omega_{i,n}\eta_i^{\frac{1}{\alpha}}b_1^+(\eta_i^{-1})\widetilde{M}_i^+>u)=H^+(u),\\
    \lim_{n\to+\infty}\sum_{i=1}^n\PB(\omega_{i,n}\eta_i^{\frac{1}{\alpha}}b_1^+(\eta_i^{-1})\widetilde{M}_i^->u)=H^-(u).
\end{align*}
Thus, we have:
\begin{align*}
    \lim_{n\to+\infty}\sum_{i=1}^n\PB(\omega_{i,n}\eta_i^{\frac{1}{\alpha}}b_1^+(\eta_i^{-1})M_n>u)=2H^+(u)\PB(M_n>0),\\
    \lim_{n\to+\infty}\sum_{i=1}^n\PB(\omega_{i,n}\eta_i^{\frac{1}{\alpha}}b_1^+(\eta_i^{-1})M_n<-u)=2H^-(u)\PB(M_n<0).
\end{align*}
Next, we evaluate the condition (2) in Lemma~\ref{lem: infinitesimal}. For any given $\varepsilon>0$, we have:
\begin{align*}
    \EB(M_i^+)^2\1\left(0\le M_i^+<\frac{\varepsilon}{\omega_{i,n}\eta_i^{\frac{1}{\alpha}}b_1^+(\eta_i^{-1})}\right)&\le \varepsilon^2+\EB(M_i^+)^2\1\left(\varepsilon< M_i^+<\frac{\varepsilon}{\omega_{i,n}\eta_i^{\frac{1}{\alpha}}b_1^+(\eta_i^{-1})}\right)\\
    &=\varepsilon^2+\int_{\rv\varepsilon^2}^{\rv\left(\frac{\varepsilon}{\omega_{i,n}\eta_i^{\frac{1}{\alpha}}b_1^+(\eta_i^{-1})}\right)^2}\PB(M_i^+>\sqrt{t})dt\notag\\
    &=\varepsilon^2+\int_{\rv\varepsilon}^{\rv\frac{\varepsilon}{\omega_{i,n}\eta_i^{\frac{1}{\alpha}}b_1^+(\eta_i^{-1})}}2t\PB(M_i^+>t)dt\notag\\
    &=\varepsilon^2+2\int_{\rv\varepsilon}^{\rv\frac{\varepsilon}{\omega_{i,n}\eta_i^{\frac{1}{\alpha}}b_1^+(\eta_i^{-1})}}\frac{b_0^+(t)}{t^{\alpha-1}}dt.
\end{align*}
When $n$ is large, $\sup_{i\le n}\omega_{i,n}\eta_i^{\frac{1}{\alpha}}b_1^+(\eta_i^{-1})\to0$. By Karamata's theorem, we have:
\begin{align*}
    &\EB(M_i^+)^2\1\left(0\le M_i^+<\frac{\varepsilon}{\omega_{i,n}\eta_i^{\frac{1}{\alpha}}b_1^+(\eta_i^{-1})}\right)\\
    \le &\varepsilon^2+C \left(\frac{\varepsilon}{\omega_{i,n}\eta_i^{\frac{1}{\alpha}}b_1^+(\eta_i^{-1})}\right)^{\rv 2-\alpha}b_0^+\left({\rv\frac{\varepsilon}{\omega_{i,n}\eta_i^{\frac{1}{\alpha}}b_1^+(\eta_i^{-1})}}\right).\notag
\end{align*}
Thus, we have:
\begin{align}
\label{eq: 2nd_decomp}
    &\sum_{i=1}^n\left(\omega_{i,n}\eta_i^{\frac{1}{\alpha}}b_1^+(\eta_i^{-1})\right)^2\EB(M_i^+)^2\1\left(0\le M_i^+<\frac{\varepsilon}{\omega_{i,n}\eta_i^{\frac{1}{\alpha}}b_1^+(\eta_i^{-1})}\right)\\
    \le&\sum_{i=1}^n\left(\omega_{i,n}\eta_i^{\frac{1}{\alpha}}b_1^+(\eta_i^{-1})\right)^2\varepsilon^2+C\varepsilon^{\rv2-\alpha}\sum_{i=1}^n\left(\omega_{i,n}\eta_i^{\frac{1}{\alpha}}b_1^+(\eta_i^{-1})\right)^{\rv\alpha}b_0^+\left({\rv\frac{\varepsilon}{\omega_{i,n}\eta_i^{\frac{1}{\alpha}}b_1^+(\eta_i^{-1})}}\right).\notag
\end{align}
For the first term in \eqref{eq: 2nd_decomp} RHS, by Stolze's theorem, we have:
\begin{align*}
    \lim_{n\to\infty}\sum_{i=1}^n\left(\omega_{i,n}\eta_i^{\frac{1}{\alpha}}b_1^+(\eta_i^{-1})\right)^2=\lim_{n\to\infty}\frac{\eta_n^{\frac{2}{\alpha}}b_1^+(\eta_n^{-1})^2}{1-(1-\eta_n)^2}=0.
\end{align*}
{\rv For the second term in \eqref{eq: 2nd_decomp} RHS, we notice that $\PB(M_n^+>x^{\frac{1}{\alpha}}b_1^{+}(x))=x^{-1}$, which implies $b_0^+\left(x^{\frac{1}{\alpha}}b_1^+(x)\right)=b_1^+(x)^\alpha$. Then we have:
\begin{align*}
    &C\varepsilon^{\rv2-\alpha}\sum_{i=1}^n\left(\omega_{i,n}\eta_i^{\frac{1}{\alpha}}b_1^+(\eta_i^{-1})\right)^{\rv\alpha}b_0^+\left({\rv\frac{\varepsilon}{\omega_{i,n}\eta_i^{\frac{1}{\alpha}}b_1^+(\eta_i^{-1})}}\right)\\
    =&C\varepsilon^{\rv2-\alpha}\sum_{i=1}^n\omega_{i,n}^{\alpha}\eta_ib_1^+(\eta_i^{-1})^{\alpha}b_0^+\left({\frac{\varepsilon}{\omega_{i,n}\eta_i^{\frac{1}{\alpha}}b_1^+(\eta_i^{-1})}}\right)\\
    =&C\varepsilon^{\rv2-\alpha}\sum_{i=1}^n\omega_{i,n}^{\alpha}\eta_i\frac{b_0^+\left({\frac{\varepsilon}{\omega_{i,n}\eta_i^{\frac{1}{\alpha}}b_1^+(\eta_i^{-1})}}\right)}{b_0^+\left(\frac{1}{\eta_i^{\frac{1}{\alpha}}b_1^+(\eta_i^{-1})}\right)}\\
    \lesssim &\varepsilon^{\rv2-\alpha}\sum_{i=1}^n\omega_{i,n}^{\alpha-\delta}\eta_i,
\end{align*}
where the last inequality holds by Potter's bound. And also we have $\sum_{i=1}^n\omega_{i,n}^{\alpha-\delta}\eta_i=\OM(1)$ by Stolze theorem.
}
Thus,  we conclude that:
\begin{align*}    {\rv\lim_{\varepsilon\to0}}\lim_{n\to\infty}\sum_{i=1}^n\left(\omega_{i,n}\eta_i^{\frac{1}{\alpha}}b_1^+(\eta_i^{-1})\right)^2\EB(\varepsilon_i^+)^2\1\left(0\le \varepsilon_i^+<\frac{\varepsilon}{\omega_{i,n}\eta_i^{\frac{1}{\alpha}}b_1^+(\eta_i^{-1})}\right)=0.
\end{align*}

As $M^-_i$ has a lighter tail than $M_i^+$, by a similar argument, we can also obtain:
\begin{align*}
    {\rv \lim_{\varepsilon\to0}}\lim_{n\to\infty}\sum_{i=1}^n\left(\omega_{i,n}\eta_i^{\frac{1}{\alpha}}b_1^+(\eta_i^{-1})\right)^2\EB(M_i^-)^2\1\left(0\le M_i^-<\frac{\varepsilon}{\omega_{i,n}\eta_i^{\frac{1}{\alpha}}b_1^+(\eta_i^{-1})}\right)=0.
\end{align*}
Thus, the condition (2) in Lemma~\ref{lem: finite_rate} is satisfied for $\sum_{i=1}^n\omega_{i,n}\eta_i^{\frac{1}{\alpha}}b_1^+(\eta_i^{-1})M_i$, which means it is tight. Thus, we conclude $\eta_n^{\frac{1}{\alpha}-1}b_1^+(\eta_n^{-1})(\theta_n-\theta^*)$ is tight. As $b_1^+(x)\ge b_1(x)$ as $b_0(x)$ can be chosen by the most heavy component, we also conclude $\eta_n^{\frac{1}{\alpha}-1}b_1(\eta_n^{-1})(\theta_n-\theta^*)$ is tight.

\subsection{Proof of Proposition~\ref{prop: z_stat_prop}}
From equation~\eqref{eqn: decay_sde}, we can write the explicit solution of $Z_t$ as:
\begin{align*}
    Z_t=\exp(-\nabla^2\ell(\theta^*)t)\left(Z_0+\int_0^t\exp(\nabla^2\ell(\theta^*)s)dL_s\right).
\end{align*}
By Proposition 2.3 of \cite{lindner2005Levy}, we have:
\begin{align*}
    \int_0^t\exp(-\nabla^2\ell(\theta^*)(t-s))dL_s\overset{d}{=}\int_0^t\exp(-\nabla^2\ell(\theta^*)s)dL_s.
\end{align*}
Thus, we have:
\begin{align*}
    Z_t\overset{d}{=}\exp(-\nabla^2\ell(\theta^*)t)Z_0+\int_0^t\exp(-\nabla^2\ell(\theta^*)s)dL_s.
\end{align*}
By Proposition 2.4 of \cite{lindner2005Levy}, if we have $\int_{\|x\|> 1}\ln\|x\|\nu(\theta^*,dx)<+\infty$, we conclude $Z_t$ converges a.s. to $\int_0^{+\infty}\exp(-\nabla^2\ell(\theta^*)t)dL_t$, which implies equation~\eqref{eqn: z_stationary} is well-defined. In fact, we have:
\begin{align*}
    \int_{\|x\|> 1}\ln\|x\|\nu(\theta^*,dx)&\propto\int_{r>1}\frac{\ln r}{r^{\alpha+1}}dr.
\end{align*}
As $\alpha\in(1,2)$, we conclude $\int_{\|x\|> 1}\ln\|x\|\nu(\theta^*,dx)<+\infty$.

\subsection{Proof of Corollary~\ref{thm: asymp_decay_lr}}
Similar with the proof of Theorem~\ref{thm: asymp_decay}, we denote $X_n:=\eta_n^{\frac{1}{\alpha}-1}b_1(\eta_n^{-1})(\theta_n-\theta^*)$ as the scaled error. By Lemma~\ref{lem: eta_expansion}, where we denote $\widetilde{c}_n:=\frac{1-\frac{1}{\alpha}}{c}a_n$ and $\widetilde{c}=\frac{1-\frac{1}{\alpha}}{c}$, we have:
\begin{align*}
    X_{n+1}=(1+\widetilde{c}_n\eta_n)\left((I-\eta_n \nabla^2\ell(\theta^*))X_n+\eta_n^{\frac{1}{\alpha}}b_1(\eta_n^{-1})M(\theta_n,\xii_{n+1})+\eta_n^{\frac{1}{\alpha}}b_1(\eta_n^{-1})R_n\right).
\end{align*}
Then, we construct two sequences $\{\widetilde{X}_n\}_{n\ge1}$ and $\{\bar{X}_n\}_{n\ge1}$ satisfying:
\begin{align*}
    &\widetilde{X}_{n+1}=(I-\eta_n (\nabla^2\ell(\theta^*)-\widetilde{c}I))\widetilde{X}_n+\eta_n^{\frac{1}{\alpha}}b_1(\eta_n^{-1})M(\theta^*,\xii_{n+1}).\\
    &\bar{X}_{n+1}=(I-\eta_n (\nabla^2\ell(\theta^*)-\widetilde{c}_nI))\bar{X}_n+\eta_n^{\frac{1}{\alpha}}b_1(\eta_n^{-1})M(\theta^*,\xii_{n+1}).
\end{align*}
By a similar argument in Theorem~\ref{thm: asymp_decay}, we have $\widetilde{X}_{n}$ weakly converges to the stationary distribution of following s.d.e.:
\begin{align*}
    dZ_t = -(\nabla^2\ell(\theta^*)-\widetilde{c})Z_t dt+dL_t.
\end{align*}
In the following, we are going to show $\widetilde{X}_n-\bar{X}_{n}$ and $\bar{X}_n-X_n$ converges to 0 in probability, which will conclude our final result.
\begin{lemma}
    We have $\widetilde{X}_n-\bar{X}_{n}$ converges to 0 in probability.
\end{lemma}
\begin{proof}
    We only need to prove this result in one dimensional case. For the convenience of notation, we denote:
    \begin{align*}
        &\widetilde{X}_{n+1}=(1-\eta_n h)\widetilde{X}_n+\eta_n^{\frac{1}{\alpha}}b_1(\eta_n^{-1})M_n.\\
        &\bar{X}_{n+1}=(1-\eta_n h_n)\bar{X}_n+\eta_n^{\frac{1}{\alpha}}b_1(\eta_n^{-1})M_n.
    \end{align*}
    We also denote $\delta_n = \widetilde{X}_{n}-\bar{X}_n$ and have:
    \begin{align*}
        \delta_{n+1}=(1-\eta_n h)\delta_n+ \eta_n(h_n-h)\bar{X}_n.
    \end{align*}
    Thus, we can write $\delta_{n}$ as:
    \begin{align*}
        \delta_n = \frac{\sum_{i=1}^{n-1}\prod_{j=1}^i(1-\eta_j h)^{-1}\eta_i(h_i-h)\bar{X}_i}{\prod_{i=1}^{n-1}(1-\eta_i h)^{-1}}:=\sum_{i=1}^{n-1}\omega_{i,n-1}\eta_i(h_i-h)\bar{X}_i.
    \end{align*}
    We also have an expression for $\bar{X}_n$ as:
    \begin{align*}
        \bar{X}_n=\frac{\sum_{i=1}^{n-1}\prod_{j=1}^i(1-\eta_j h_j)^{-1}\eta_i^{\frac{1}{\alpha}}b_1(\eta_i^{-1})M_i}{\prod_{i=1}^{n-1}(1-\eta_i h_i)^{-1}}:=\sum_{i=1}^{n-1}\widetilde{\omega}_{i,n-1}\eta_i^{\frac{1}{\alpha}}b_1(\eta_i^{-1})M_i.
    \end{align*}
    Then, we can re-write $\delta_n$ as:
    \begin{align*}
        \delta_n &= \sum_{i=1}^{n-1}\omega_{i,n-1}\eta_i(h_i-h)\sum_{j=1}^{i-1}\widetilde{\omega}_{j,i-1}\eta_j^{\frac{1}{\alpha}}b_1(\eta_j^{-1})M_j\\
        &=\sum_{j=1}^{n-2}\sum_{i=j+1}^{n-1}\omega_{i,n-1}\eta_i(h_i-h)\widetilde{\omega}_{j,i-1}\eta_j^{\frac{1}{\alpha}}b_1(\eta_j^{-1})M_j\notag\\
        &:=\sum_{j=1}^{n-2}\lambda_{j,n-2}\eta_j^{\frac{1}{\alpha}}b_1(\eta_j^{-1})M_j.\notag
    \end{align*}
    Then, by Lemma~\ref{lem: infinitesimal}, for any $u>0$, we aim to verify (reverse is also true):
    \begin{align*}
        \lim_{n\to\infty}\sum_{j=1}^{n-2}\PB(\lambda_{j,n-2}\eta_j^{\frac{1}{\alpha}}b_1(\eta_j^{-1})M_j>u)=0.
    \end{align*}
    In fact, we have:
    \begin{align*}
        \PB(\lambda_{j,n-2}\eta_j^{\frac{1}{\alpha}}b_1(\eta_j^{-1})M_j>u)\lesssim\frac{\lambda_{j,n-2}^{\alpha-\delta}\eta_j}{u^{\alpha-\delta}},
    \end{align*}
    where we use the Potter bounds as in Eqn~\eqref{eqn: potter} and $\alpha-\delta>1$. Then we also notice that:
    \begin{align*}
        |\lambda_{j,n-2}|\le\frac{j+1}{n-1}\sum_{i=j+1}^{n-1}\frac{\eta_i}{1-h\eta_i}|h_i-h|,
    \end{align*}
    where we use the fact that $(1-h\eta_{i})\eta_{i-1}\le\eta_i$ whenever $ch\ge 1$ and $\eta_i=c\cdot i^{-1}$. Moreover, there may be an issue that $1-h\eta_i<0$. In this case, we can adjust the learning rate to $\eta_i=c\cdot (i+z)^{-1}$ with some $z>0$, which has no influence on the result. Then, we have:
    \begin{align*}
        |\lambda_{j,n-2}|&\lesssim\frac{j+1}{n-1}\sum_{i=j+1}^{n-1}\frac{|h_i-h|}{i}\\
        &\lesssim\frac{j+1}{n-1}\left((n-j)^{\alpha-\delta-1}\sum_{i=j+1}^{n-1}\frac{|h_i-h|^{\alpha-\delta}}{i^{\alpha-\delta}}\right)^{\frac{1}{\alpha-\delta}},\notag
    \end{align*}
    where the last inequality holds by Jensen's inequality. Then, we have:
    \begin{align*}
        \sum_{j=1}^{n-2}\lambda_{j,n-2}^{\alpha-\delta}\eta_j&\lesssim\frac{\sum_{j=1}^{n-2}j^{\alpha-\delta-1}\sum_{i=j+1}^n\frac{|h_i-h|^{\alpha-\delta}}{i^{\alpha-\delta}}}{n-1}\\
        &=\frac{\sum_{i=1}^{n-1}\sum_{j=1}^{i-1}j^{\alpha-\delta-1}\frac{|h_i-h|^{\alpha-\delta}}{i^{\alpha-\delta}}}{n-1}\notag\\
        &\lesssim\frac{\sum_{i=1}^{n-1}|h_i-h|^{\alpha-\delta}}{n-1}\to0.\notag
    \end{align*}
    Thus, $\delta_n\overset{d}{\to} 0$, which implies $\delta_n$ converges to 0 in probability and our lemma is concluded.
\end{proof}

\begin{lemma}
    We have $\bar{X}_{n}-X_n$ converges to 0 in probability.
\end{lemma}
\begin{proof}
    We denote $\delta_n=X_{n}-\bar{X}_n$, and have:
    \begin{align*}
        \delta_{n+1}=&(I-\eta_n(\nabla^2\ell(\theta^*)-\widetilde{c}_n I))\delta_n-\widetilde{c}_n\eta_n^2\nabla^2\ell(\theta^*)X_n+\widetilde{c}_n\eta_n^{\frac{1}{\alpha}+1}M(\theta_n,\xii_{n+1})\\
        &+\eta_n^{\frac{1}{\alpha}}b_1(\eta_n^{-1})(M(\theta_n,\xii_{n+1})-M(\theta^*,\xii_{n+1}))\notag\\
        &+(1+\widetilde{c}_n\eta_n)\eta_n^{\frac{1}{\alpha}}b_1(\eta_n^{-1})R_n.\notag
    \end{align*}
    Then, we can decompose $\delta_n$ into four terms: $\delta_n= \sum_{i=1}^4\delta_n^{(i)}$, where
    \begin{align*}
        \delta_{n+1}^{(1)}=&(I-\eta_n(\nabla^2\ell(\theta^*)-\widetilde{c}_n I))\delta_n^{(1)}-\widetilde{c}_n\eta_n^2\nabla^2\ell(\theta^*)X_n,\\
        \delta_{n+1}^{(2)}=&(I-\eta_n(\nabla^2\ell(\theta^*)-\widetilde{c}_n I))\delta_n^{(2)}+\widetilde{c}_n\eta_n^{\frac{1}{\alpha}+1}M(\theta_n,\xii_{n+1}),\\
        \delta_{n+1}^{(3)}=&(I-\eta_n(\nabla^2\ell(\theta^*)-\widetilde{c}_n I))\delta_n^{(3)}+\eta_n^{\frac{1}{\alpha}}b_1(\eta_n^{-1})(M(\theta_n,\xii_{n+1})-M(\theta^*,\xii_{n+1})),\\
        \delta_{n+1}^{(4)}=&(I-\eta_n(\nabla^2\ell(\theta^*)-\widetilde{c}_n I))\delta_n^{(4)}+(1+\widetilde{c}_n\eta_n)\eta_n^{\frac{1}{\alpha}}b_1(\eta_n^{-1})R_n.
    \end{align*}
    For $\delta_{n}^{(1)}$, we only need to notice that $\eta_n^2\EB\|X_n\|=\OM\left(\eta_n^{\frac{1}{\alpha}-\frac{1}{p}+2}b_1(\eta_n^{-1})\right)$. For $\delta_{n}^{(2)}$, we only need to notice $\EB\| M(\theta_n,\xii_{n+1})\|<+\infty$ as $\alpha>1$. For $\delta_n^{(3)}$ and $\delta_n^{(4)}$, we can prove they converge to 0 in probability in a same way as in proof of Theorem~\ref{thm: asymp_decay}. Then, by Lemma~\ref{lem: recursion}, we can verify that $\delta_n$ converges to 0 in probability.
\end{proof}

\subsection{Proof of Theorem~\ref{thm: ols_norm}}
In fact, we have:
\begin{align*}
    \frac{\PB(\|\nabla \ell(\theta,x_i,\varepsilon_i)\|>tz)}{\PB(\|\nabla \ell(\theta,x_i,\varepsilon_i)\|>z)}=\frac{\int\PB(\|\nabla \ell(\theta,x,\varepsilon_i)\|>tz)\mu(dx)}{\int\PB(\|\nabla \ell(\theta,x,\varepsilon_i)\|>z)\mu(dx)}.
\end{align*}
As $x_i$ is generated from a compact set, by mean-value theorem, there exists $\widehat{x}$, which depends with $z$,  such that:
\begin{align*}
    \frac{\PB(\|\nabla \ell(\theta,x_i,\varepsilon_i)\|>tz)}{\PB(\|\nabla \ell(\theta,x_i,\varepsilon_i)\|>z)}=\frac{\PB(\|\nabla \ell(\theta,\widehat{x},\varepsilon_i)\|>tz)}{\PB(\|\nabla \ell(\theta,\widehat{x},\varepsilon_i)\|>z)}.
\end{align*}
We also notice that:
\begin{align*}
    \PB(\|\nabla \ell(\theta,\widehat{x},\varepsilon_i)\|>z)=\PB(|\varepsilon_i-\widehat{x}^\top(\theta-\theta^*)|>z/\|\widehat{x}\|).
\end{align*}
As $\varepsilon_i$ is $\alpha$-regular varying and $\widehat{x}$ locates in a compact set, we deduce that:
\begin{align*}
    \lim_{z\to\infty}\frac{\PB(\|\nabla \ell(\theta,\widehat{x},\varepsilon_i)\|>tz)}{\PB(\|\nabla \ell(\theta,\widehat{x},\varepsilon_i)\|>z)}=t^{-\alpha}.
\end{align*}
Moreover, we denote $\Theta(\cdot)=\PB\left(\frac{x_i}{\|x_i\|}\in\cdot\right)$. For any Borel set $A$ such that $\Theta(\partial A)=0$, we have:
\begin{align*}
    &\PB\left(\left.\frac{x_i}{\|x_i\|}\text{sgn}(x_i^\top(\theta-\theta^*)-\varepsilon_i)\in A\right|\|x_i\||x_i^\top(\theta-\theta^*)-\varepsilon_i|>z\right)\\
    =&\frac{\PB\left(\frac{x_i}{\|x_i\|}\in A, \|x_i\|(x_i^\top(\theta-\theta^*)-\varepsilon_i)>z\right)}{\PB\left(\|x_i\||x_i^\top(\theta-\theta^*)-\varepsilon_i|>z\right)}+\frac{\PB\left(\frac{-x_i}{\|x_i\|}\in A, \|x_i\|(x_i^\top(\theta-\theta^*)-\varepsilon_i)<-z\right)}{\PB\left(\|x_i\||x_i^\top(\theta-\theta^*)-\varepsilon_i|>z\right)}\notag.
\end{align*}
Furthermore, we denote:
\begin{align*}
    h(x,z) = \frac{\PB(\|x\|(x^\top(\theta-\theta^*)-\varepsilon_i)>z)}{\int \PB(\|x\|(x^\top(\theta-\theta^*)-\varepsilon_i)>z)\mu(dx)}.
\end{align*}
We claim $\lim_{z\to\infty} h(x,z)=h(x)$ uniformly in $x$. In fact, as $x$ locates at a compact set, we have:
\begin{align*}
    \lim_{z\to\infty}z^{\alpha}\PB(\|x\|(x^\top(\theta-\theta^*)-\varepsilon_i)>z)=\|x\|^\alpha,
\end{align*}
which is uniformly convergence. Thus, we have:
\begin{align*}
    \lim_{z\to\infty}z^\alpha\int \PB(\|x\|(x^\top(\theta-\theta^*)-\varepsilon_i)>z)\mu(dx)=\EB\|x_i\|^\alpha.
\end{align*}
Thus, we have:
\begin{align*}
    \lim_{z\to\infty}h(x,z)=\frac{\|x\|^\alpha}{\EB\|x_i\|^\alpha},
\end{align*}
which is uniformly in $x$. By similar arguments, we have:
\begin{align*}
    &\lim_{z\to\infty}\frac{\PB(\|x\|(x^\top(\theta-\theta^*)-\varepsilon_i)>z)}{\int \PB(\|x\|(x^\top(\theta-\theta^*)-\varepsilon_i)<-z)\mu(dx)}\notag\\
    =&    \lim_{z\to\infty}\frac{\PB(\|x\|(x^\top(\theta-\theta^*)-\varepsilon_i)<-z)}{\int \PB(\|x\|(x^\top(\theta-\theta^*)-\varepsilon_i)>z)\mu(dx)}\notag\\
    =&    \lim_{z\to\infty}\frac{\PB(\|x\|(x^\top(\theta-\theta^*)-\varepsilon_i)<-z)}{\int \PB(\|x\|(x^\top(\theta-\theta^*)-\varepsilon_i)<-z)\mu(dx)}\notag\\
    =&\frac{\|x\|^\alpha}{\EB\|x_i\|^\alpha}.\notag
\end{align*}
Thus, we have:
\begin{align*}
    &\lim_{t\to\infty}\PB\left(\left.\frac{x_i}{\|x_i\|}\text{sgn}(x_i^\top(\theta-\theta^*)-\varepsilon_i)\in A\right|\|x_i\||x_i^\top(\theta-\theta^*)-\varepsilon_i|>z\right)\\
    =&\frac{\EB\1_A(\frac{x_i}{\|x_i\|})\|x_i\|^\alpha+\EB\1_A(\frac{-x_i}{\|x_i\|})\|x_i\|^\alpha}{2\EB\|x_i\|^{\alpha}}.\notag
\end{align*}

\subsection{Proof of Theorem~\ref{thm: logistic_measure}}
Firstly, we consider $\lambda=0$. To simplify the notation, we denote $\sigma(x) = \frac{\exp(-x)}{1+\exp(-x)}$. We denote $h_i := y_i x_i$. First, we prove the following measure vaguely converges:
\begin{align}
\label{eq: cond_meas}
    \PB\left(\left.\frac{-h_i}{\|h_i\|}\in\cdot\right|\|h_i\|>z\right).
\end{align}
In fact, for any Borel set $A$ such that $\mu(\partial A)=0$, we have:
\begin{align*}
    \PB\left(\left.\frac{-h_i}{\|h_i\|}\in A\right|\|h_i\|>z\right)=\PB\left(\left.\frac{-y_ix_i}{\|x_i\|}\in A\right|\|x_i\|>z\right).
\end{align*}
And, we have:
\begin{align*}
    \PB\left(\left.Y_i=1,x_i^\top\theta^*>0,\frac{-x_i}{\|x_i\|}\in A\right|\|x_i\|>z\right)=\int_{\frac{-x}{\|x\|}\in A, x^\top\theta^*>0}\sigma(-x^\top\theta^*)\mu_z(dx),
\end{align*}
where $\mu_z(dx)=\PB(x_i\in dx|\|x_i\|>z)$. Letting $z\to\infty$, we have:
\begin{align*}
    \limsup_{z\to\infty}\PB\left(\left.Y_i=1,x_i^\top\theta^*>0,\frac{-x_i}{\|x_i\|}\in A\right|\|x_i\|>z\right)\le\mu((-A)\cap\{x\in\mathbb{S}^{d-1}|x^\top\theta^*>0\}).
\end{align*}
On the other side, for any $\varepsilon>0$, there exists $M$ and $\delta$ ($\delta\to0$ as $\varepsilon\to 0$), such that $\sigma(-M\delta)>\frac{1}{1+\varepsilon}$. Then, for $z>M$, we have:
\begin{align*}
    \int_{\frac{-x}{\|x\|}\in A, x^\top\theta^*>0}\sigma(-x^\top\theta^*)\mu_z(dx)\ge\int_{\frac{-x}{\|x\|}\in A, x^\top\theta^*>\delta\|x\|}\frac{1}{1+\varepsilon}\mu_z(dx).
\end{align*}
Thus, we have:
\begin{align*}
    \liminf_{z\to\infty}\int_{\frac{-x}{\|x\|}\in A, x^\top\theta^*>0}\sigma(-x^\top\theta^*)\mu_z(dx)\ge\frac{\mu((-A)\cap\{x\in\mathbb{S}^{d-1}|x^\top\theta^*>\delta\})}{1+\varepsilon}.
\end{align*}
Letting $\varepsilon\to0$, we conclude:
\begin{align*}
    \lim_{z\to\infty}\PB\left(\left.Y_i=1,x_i^\top\theta^*>0,\frac{-x_i}{\|x_i\|}\in A\right|\|x_i\|>z\right)=\mu((-A)\cap\{x\in\mathbb{S}^{d-1}|x^\top\theta^*>0\}).
\end{align*}
We also have:
\begin{align*}
    &\PB\left(\left.Y_i=1,x_i^\top\theta^*<0,\frac{-x_i}{\|x_i\|}\in A\right|\|x_i\|>z\right)\\
    =&\int_{\frac{-x}{\|x\|}\in A, x^\top\theta^*<0}\sigma(-x^\top\theta^*)\mu_z(dx)\notag\\
    \le&\int_{\frac{-x}{\|x\|}\in A, x^\top\theta^*<0}\exp(x^\top\theta^*)\mu_z(dx)\notag\\
    \le&\int_{\frac{-x}{\|x\|}\in A, x^\top\theta^*<0}\frac{\beta_2^{-\beta_2}\exp(-\beta_2)}{(-x^\top\theta^*)^{\beta_2}}\mu_z(dx).\notag\\
    \le&\frac{1}{z^{\beta_2}}\int_{x^\top\theta^*<0}\frac{\|x\|^{\beta_2}\beta_2^{-\beta_2}\exp(-\beta_2)}{(-x^\top\theta^*)^{\beta_2}}\mu_z(dx).\notag
\end{align*}
By the fact that $\lim_{z\to\infty}\int_{ x^\top\theta^*<0}\left(\frac{\|x\|}{-x^\top\theta^*}\right)^{\beta_2}\mu_z(dx)<+\infty$,  letting $z\to\infty$, we have:
\begin{align*}
    \limsup_{z\to\infty}\PB\left(\left.Y_i=1,x_i^\top\theta^*<0,\frac{-x_i}{\|x_i\|}\in A\right|\|x_i\|>z\right)=0.
\end{align*}
Thus, we have:
\begin{align*}
    \lim_{z\to\infty}\PB\left(\left.Y_i=1,\frac{-x_i}{\|x_i\|}\in A\right|\|x_i\|>z\right)=\mu((-A)\cap\{x\in\mathbb{S}^{d-1}|x^\top\theta^*>0\}).
\end{align*}
By a similar argument, we also have:
\begin{align*}
    \lim_{z\to\infty}\PB\left(\left.Y_i=-1,\frac{x_i}{\|x_i\|}\in A\right|\|x_i\|>z\right)=\mu(A\cap\{x\in\mathbb{S}^{d-1}|x^\top\theta^*<0\}).
\end{align*}
Thus, we have:
\begin{align}
\label{eq: mu}
    &\lim_{z\to\infty}\PB\left(\left.\frac{-h_i}{\|h_i\|}\in\cdot\right|\|h_i\|>z\right)\overset{v}{:=}\nu(\cdot)\\
    =&\mu((-\cdot)\cap\{x\in\mathbb{S}^{d-1}|x^\top\theta^*>0\})+\mu(\cdot\cap\{x\in\mathbb{S}^{d-1}|x^\top\theta^*<0\}).\notag
\end{align}
Next, we turn to study the desired conditional measure. In fact, for any Borel set $A$, it satisfies:
\begin{align*}
    &\PB\left(\left.\frac{-h_i}{\|h_i\|}\in A\right|\sigma(h_i^\top\theta)\|h_i\|>z\right)\\
    =&\frac{\PB\left(\frac{-h_i}{\|h_i\|}\in A,\sigma(h_i^\top\theta)\|h_i\|>z\right)}{\PB\left(\sigma(h_i^\top\theta)\|h_i\|>z\right)}\notag\\
    =&\frac{\PB\left(\frac{-h_i}{\|h_i\|}\in A,\sigma(h_i^\top\theta)\|h_i\|>z, h_i^\top\theta<0\right)+\PB\left(\frac{-h_i}{\|h_i\|}\in A,\sigma(h_i^\top\theta)\|h_i\|>z, h_i^\top\theta>0\right)}{\PB\left(\sigma(h_i^\top\theta)\|h_i\|>z,h_i^\top\theta<0\right)+\PB\left(\sigma(h_i^\top\theta)\|h_i\|>z,h_i^\top\theta>0\right)}.\notag
\end{align*}
As $\sigma(x)\le 1$, we have:
\begin{align*}
    &\PB\left(\left.\frac{-h_i}{\|h_i\|}\in A\right|\sigma(h_i^\top\theta)\|h_i\|>z\right)\\
    =&\frac{\PB\left(\left.\frac{-h_i}{\|h_i\|}\in A,\sigma(h_i^\top\theta)\|h_i\|>z, h_i^\top\theta<0\right|\|h_i\|>z\right)+\PB\left(\left.\frac{-h_i}{\|h_i\|}\in A,\sigma(h_i^\top\theta)\|h_i\|>z, h_i^\top\theta>0\right|\|h_i\|>z\right)}{\PB\left(\left.\sigma(h_i^\top\theta)\|h_i\|>z,h_i^\top\theta<0\right|\|h_i\|>z\right)+\PB\left(\left.\sigma(h_i^\top\theta)\|h_i\|>z,h_i^\top\theta>0\right|\|h_i\|>z\right)}\notag
\end{align*}
We claim that:
\begin{align}
\label{claim: neg_equi}
    \lim_{z\to\infty}\PB\left( \left.\sigma(h_i^\top\theta)\|h_i\|>z,h_i^\top\theta<0\right|\|h_i\|>z\right)=\nu(\{x\in\mathbb{S}^{d-1}|x^\top\theta>0\}).
\end{align}
As $\sigma(x)\le 1$, by~\eqref{eq: mu}, we have:
\begin{align*}
    \limsup_{z\to\infty}\PB\left( \left.\sigma(h_i^\top\theta)\|h_i\|>z,h_i^\top\theta<0\right|\|h_i\|>z\right)\le \nu(\{x\in\mathbb{S}^{d-1}|x^\top\theta>0\}).
\end{align*}
For any $\varepsilon>0$, there exists $M$ and $\delta$ ($\delta\to0$ as $\varepsilon\to0$), such that $\sigma(-M\delta)>\frac{1}{1+\varepsilon}$. Then, when $z>M$, we have:
\begin{align*}
    \left\{\|h_i\|>(1+\varepsilon)z, h_i^\top\theta<-\delta\|h_i\|\right\}\subseteq\left\{\sigma(h_i^\top\theta)\|h_i\|>z, h_i^\top\theta<0\right\}.
\end{align*}
Then, we have:
\begin{align*}
    \lim_{z\to\infty}\PB\left(\left.\|h_i\|>(1+\varepsilon)z,h_i^\top\theta<-\delta\|h_i\|\right|\|h_i\|>z\right)=\frac{\nu(\{x\in\mathbb{S}^{d-1}|x^\top\theta>\delta\})}{(1+\varepsilon)^\alpha}.
\end{align*}
Then, we have the claim~\eqref{claim: neg_equi} holding by letting $\varepsilon\to0$. Next, we claim that:
\begin{align}
    \label{claim: pos_0}
    \lim_{z\to\infty}\PB\left(\left.\sigma(h_i^\top\theta)\|h_i\|>z,h_i^\top\theta>0\right|\|h_i\|>z\right)=0.
\end{align}
We notice that when $x^\top\theta>0$:
\begin{align*}
    \sigma(x^\top\theta)\le\exp(-x^\top\theta)\le\frac{\exp(-1)}{x^\top\theta}.
\end{align*}
Thus, we have:
\begin{align*}
    \PB\left(\left.\sigma(h_i^\top\theta)\|h_i\|>z,h_i^\top\theta>0\right|\|h_i\|>z\right)\le\PB\left(\left.0<\frac{h_i^\top\theta}{\|h_i\|}<\frac{\exp(-1)}{z}\right|\|h_i\|>z\right).
\end{align*}
For any $\varepsilon>0$, there exists $z_0$, s.t. for all $z>z_0$, we have:
\begin{align*}
    \PB\left(\left.\sigma(h_i^\top\theta)\|h_i\|>z,h_i^\top\theta>0\right|\|h_i\|>z\right)\le\PB\left(\left.0<\frac{h_i^\top\theta}{\|h_i\|}<\varepsilon\right|\|h_i\|>z\right).
\end{align*}
Letting $z\to\infty$, we have:
\begin{align*}
    \lim_{z\to\infty}\PB\left(\left.\sigma(h_i^\top\theta)\|h_i\|>z,h_i^\top\theta>0\right|\|h_i\|>z\right)\le\nu(\{x\in\mathbb{S}^{d-1}|x^\top\theta\in(-\varepsilon,0)\}).
\end{align*}
As $\varepsilon$ is arbitrary, we have claim~\eqref{claim: pos_0} holding. By these two claims~\eqref{claim: neg_equi} and \eqref{claim: pos_0}, for any Borel set $A$, we have:
\begin{align}
\label{eq: nu}
    \lim_{z\to\infty}\PB\left(\left.\frac{-h_i}{\|h_i\|}\in A\right|\sigma(x_i^\top\theta)\|h_i\|>z\right)
    =\frac{\nu(A\cap\{x\in\mathbb{S}^{d-1}|x^\top\theta>0\})}{\nu(\{x\in\mathbb{S}^{d-1}|x^\top\theta>0\})}.
\end{align}
Thus, the measure~\eqref{eq: cond_meas} vaguely converges by~\eqref{eq: mu} and \eqref{eq: nu}. Moreover, by claim~\eqref{claim: neg_equi} and \eqref{claim: pos_0}, we have:
\begin{align}
    \lim_{z\to\infty}\frac{\PB\left(\sigma(h_i^\top\theta)\|h_i\|>z\right)}{\PB\left(\|h_i\|>z\right)}=\nu(\{x\in\mathbb{S}^{d-1}|x^\top\theta>0\}).
\end{align}
Thus, the tail distribution of $\sigma(h_i^\top\theta)\|h_i\|$ is also $\alpha$-regular varying as $\|h_i\|=\|x_i\|$ is $\alpha$-regular varying. For the general case $\lambda>0$, we have $\nabla\ell(\theta,x_i,y_i)=-h_i\sigma(h_i^\top\theta)+\lambda\theta$. As $\lambda$ and $\theta$ are fixed parameters, the asymptotics remain the same.

\section{Conclusion}
\label{sec: conclusion}
In this paper, we study the weak convergence of SGD algorithm when the distribution of the stochastic gradient belongs to the domain of attraction of the stable law in high dimensional space. We provide the functional limit theorem when the learning rate is a constant, which reveals the weakly convergence of the sample path to the ground-truth gradient flow. We also provide the limit distribution of the SGD algorithm when the learning rate is decaying with iteration steps, which is the stationary distribution of an {\rv Ornstein-Uhlenbeck} process driven by a L\'evy process. Finally, we discuss the applications to the linear regression and logistic regression models in the presence of heavy tail noises.  {\rv Meanwhile, our current framework relies on smoothness and strong convexity assumptions, which generally do not hold for more complex architectures such as over-parametrized neural networks. Extending the theory to scaling limits for modern statistical models is therefore highly non-trivial and requires different mathematical tools, and we view this as a promising and important direction for future research. Meanwhile, a gap remains in conducting statistical inference for heavy-tailed SGD using these limit theorems, primarily due to three challenges: (a) the unknown stability index $\alpha$; (b) the unknown slowly varying function $b_1(\cdot)$; and (c) the unknown quantiles of the limiting distributions. We leave the development of efficient methods for constructing valid confidence intervals for $\theta^*$ to future work.}

\section*{Funding}
AM was supported in part by EPSRC grants EP/V009478/1 and EP/W006227/1 and would also like to thank the Isaac Newton Institute for
Mathematical Sciences, Cambridge, for support during the INI satellite programme \textit{Heavy
Tails in Machine Learning}, funded by EPSRC grant EP/R014604/1, where work on this paper was undertaken. Jose and Wenhao acknowledge support by the Air Force Office of Scientific Research under award number FA9550-20-1-0397 and additional support is gratefully acknowledged from NSF 2118199, 2229012, 2312204, and ONR 13983111.

\bibliographystyle{plainnat}
\bibliography{refer.bib}
\end{document}